\documentclass[10pt]{article}

\usepackage[letterpaper, left=1in, top=1in, right=1in, bottom=1in, verbose, ignoremp]{geometry}

\usepackage{url}\RequirePackage[colorlinks,citecolor=blue, linkcolor=blue,urlcolor = blue]{hyperref}
\usepackage{latexsym,amssymb,amsmath,amsfonts,graphicx,color,fancyvrb,amsthm,enumerate,subfigure,mathrsfs,mathtools}
\usepackage[authoryear,round]{natbib}
\usepackage[dvipsnames]{xcolor}
\usepackage{xy}\xyoption{all} \xyoption{poly} \xyoption{knot}
\usepackage{float}
\usepackage{subfloat}
\usepackage{bm}
\usepackage{bbm}
\usepackage{multicol,multirow}
\usepackage{array}
\usepackage{relsize}
\usepackage{booktabs}
\usepackage{chngcntr}
\usepackage{etoolbox}
\usepackage{caption}
\usepackage{tikz}
\usepackage{tikz-cd}
\usetikzlibrary{patterns}
\usepackage{amsopn}
\usepackage{calc}
\usepackage{algorithm}
\usepackage[noend]{algpseudocode}
\usepackage{pdflscape}
\usepackage{multirow,multicol}

\usepackage{hyperref}
\usepackage{rotating}

\newtheorem{theorem}{Theorem}[]
\newtheorem*{theorem*}{Theorem}

\newtheorem{lemma}[theorem]{Lemma}
\newtheorem{proposition}[theorem]{Proposition}

\newtheorem*{claim*}{Claim}
\theoremstyle{definition}
\newtheorem{definition}[theorem]{Definition}
\newtheorem*{definition*}{Definition}
\theoremstyle{AppDefinition}

\theoremstyle{AppClaim}

\theoremstyle{remark}
\newtheorem{remark}[theorem]{Remark}

\newtheorem*{example*}{Example}

\def\beginmat{ \left( \begin{array} }
\def\endmat{ \end{array} \right) }

\def\log{{\rm log}}

\makeatletter
\newcommand*{\op}{%
  \DOTSB
  \mathop{\vphantom{\bigoplus}\mathpalette\matt@op\relax}%
  \slimits@
}
\newcommand\matt@op[2]{%
  \vcenter{\m@th\hbox{\resizebox{\widthof{$#1\bigoplus$}}{!}{$\boxplus$}}}%
}
\makeatother

\def\R{{\mathbb R}}

\newcommand{\pers}{\mathrm{pers}}
\newcommand{\bd}{\mathrm{bd}}

\makeatletter
\def\@biblabel#1{}
\makeatother

\makeatletter
\patchcmd{\NAT@citex}
  {\@citea\NAT@hyper@{%
     \NAT@nmfmt{\NAT@nm}%
     \hyper@natlinkbreak{\NAT@aysep\NAT@spacechar}{\@citeb\@extra@b@citeb}%
     \NAT@date}}
  {\@citea\NAT@nmfmt{\NAT@nm}%
   \NAT@aysep\NAT@spacechar\NAT@hyper@{\NAT@date}}{}{}

\patchcmd{\NAT@citex}
  {\@citea\NAT@hyper@{%
     \NAT@nmfmt{\NAT@nm}%
     \hyper@natlinkbreak{\NAT@spacechar\NAT@@open\if*#1*\else#1\NAT@spacechar\fi}%
       {\@citeb\@extra@b@citeb}%
     \NAT@date}}
  {\@citea\NAT@nmfmt{\NAT@nm}%
   \NAT@spacechar\NAT@@open\if*#1*\else#1\NAT@spacechar\fi\NAT@hyper@{\NAT@date}}
  {}{}

\makeatother

\begin{document}
\def\spacingset#1{\renewcommand{\baselinestretch}%
{#1}\small\normalsize} \spacingset{1}
\begin{flushleft}
{\Large{\textbf{Topological Information Retrieval with Dilation-Invariant Bottleneck Comparative Measures}}}
\newline
\\
Yueqi Cao$^{1,*}$, Athanasios Vlontzos$^{2,*}$, Luca Schmidtke$^{2}$, Bernhard Kainz$^{2}$, and Anthea Monod$^{1,\dagger}$
\\
\bigskip
\bf{1} Department of Mathematics, Imperial College London, UK
\\
\bf{2} Department of Computing, Imperial College London, UK
\\
\bigskip
* These authors contributed equally to this work\\
$\dagger$ Corresponding e-mail: a.monod@imperial.ac.uk
\end{flushleft}


\section*{Abstract}

Appropriately representing elements in a database so that queries may be accurately matched is a central task in information retrieval; recently, this has been achieved by embedding the graphical structure of the database into a manifold in a hierarchy-preserving manner using a variety of metrics.  Persistent homology is a tool commonly used in topological data analysis that is able to rigorously characterize a database in terms of both its hierarchy and connectivity structure.  Computing persistent homology on a variety of embedded datasets reveals that some commonly used embeddings fail to preserve the connectivity.  We show that those embeddings which successfully retain the database topology coincide in persistent homology by introducing two dilation-invariant comparative measures to capture this effect: in particular, they address the issue of metric distortion on manifolds. We provide an algorithm for their computation that exhibits greatly reduced time complexity over existing methods.  We use these measures to perform the first instance of topology-based information retrieval and demonstrate its increased performance over the standard bottleneck distance for persistent homology.  We showcase our approach on databases of different data varieties including text, videos, and medical images.\\

\noindent
{\bf Keywords:} Bottleneck distance; Database embeddings; Dilation invariance; Information retrieval; Persistent homology.


\section{Introduction}
\label{sec:intro}

The fundamental problem of information retrieval (IR) is to find the most related elements in a database for a given query.  Given that queries often comprise multiple components, candidate matches must satisfy multiple conditions, which gives rise to a natural hierarchy and connectivity structure of the database.  This structure is important to maintain in performing IR.  For example, when searching for a person named ``John Smith" in a database, the search algorithm may search among all entries with the last name ``Smith" and then among those entries, search for those with the first name ``John."  Additionally, cycles and higher order topological features in the database correspond to entries that are directly related and may be, for instance, suitable alternative matches: as an example, in online shopping, when recommendations for other products are proposed as either alternative or complementary to the original query (i.e., they are recommended for their connection).  The hierarchical and connectivity structure of databases motivates the study of the {\em topology} of databases. Topology characterizes abstract geometric properties of a set or space, such as its connectivity.  Prior work has used point-set topology to describe databases \citep{EGGHE199861, CLEMENTINI1994815, https://doi.org/10.1002/(SICI)1097-4571(1998)49:13<1144::AID-ASI2>3.0.CO;2-Z, https://doi.org/10.1002/(SICI)1097-4571(199212)43:10<658::AID-ASI3>3.0.CO;2-H}.  In this paper, we explore an alternative approach based on algebraic topology.

Topological data analysis (TDA) has recently been utilized in many fields and yielded prominent results, including imaging \citep{SECT, perea2014klein}, biology and neuroscience \citep{10.1007/978-3-030-00755-3_8, aukerman_et_al:LIPIcs:2020:12169}, materials science \citep{Hiraoka201520877, hirata2020structural}, and sensor networks \citep{doi:10.1177/0278364914548051, doi:10.1177/0278364906072252}.  TDA has also recently gained interest in machine learning (ML) in many contexts, including loss function construction, generative adversarial networks, deep learning and deep neural networks, representation learning, kernel methods, and autoencoders \citep{Bruel-Gabrielsson2019, hofer2017deep, pmlr-v97-hofer19a, Hu2019, Moor2020, reininghaus2015stable}.


{\em Persistent homology} is a fundamental TDA methodology that extracts the topological features of a dataset in an interpretable, lower-dimensional representation \citep{892133, frosini1999size, zomorodian2005computing}.  Persistent homology is particularly amenable to data analysis since it produces a multi-scale summary, known as a {\em persistence diagram}, which tracks the presence and evolution of topological features.  Moreover, persistent homology is robust to noise and allows for customization of the chosen metric, making it applicable to a wide range of applications \citep{cohen2007stability}.

In this paper, we apply persistent homology to the setting of database queries and IR to understand the structure and connectivity of databases, which are possibly unknown, but may be summarized using topology.  Persistent homology, in particular, provides an intuition on the prominence of each topological feature within the dataset.  We show that the topological characteristics of a hierarchical and interdependent database structure are preserved in some commonly used embeddings and not in others. Moreover, we find that there is a high degree of similarity between topology-preserving embeddings.  To capture and quantify this similarity, we introduce the {\em dilation-invariant bottleneck dissimilarity} and {\em distance} for persistent homology to mitigate the issue of metric distortion in manifold embeddings.  We use these measures in an exploratory analysis as well as to perform an IR task on real databases, where we show increased performance over the standard bottleneck distance for persistent homology.  To the best of our knowledge, this is the first instance where TDA tools are used to describe database representations and to perform IR.

The remainder of this paper is structured as follows.  We close this first Section \ref{sec:intro} with a discussion on related work.  We then present the key ideas behind TDA in Section \ref{sec:tda} and introduce our dilation-invariant bottleneck comparative measures.  In Section \ref{sec:methods} we present efficient methods to compute both dilation-invariant bottleneck comparative measures and prove results on the correctness and complexity of our algorithms.  In particular, we show that the time complexity of our algorithm to compute the dilation-invariant bottleneck distance is vastly more efficient than existing work that computes the related shift-invariant bottleneck distance.  In Section \ref{sec:application}, we present two applications to data: an exploratory data analysis on database representations and an IR task by classification.  We conclude in Section \ref{sec:discussion} with some thoughts from our study and directions for future research.

\paragraph{Related Work.}

IR has been a topic of active research, particularly within the natural language processing (NLP) community. Document retrieval approaches have proposed a sequence-to-sequence model 
to increase performance in cross-lingual IR \citep{boudin-etal-2020-keyphrase,liu-etal-2020-cross-lingual-document}. A particularly interesting challenge in IR in NLP is word-sense disambiguation, where the closest meaning of a word given a context is retrieved. 
Approaches to this problem have been proposed where the power of large language models is leveraged to embed not only words but also their contexts; the embeddings are further enriched with other sources, such as the WordNet hierarchical structure, before comparing the queries to database elements \citep{bevilacqua-navigli-2020-breaking,scarlini-etal-2020-sensembert,yap-etal-2020-adapting}.  Inference is usually done either using cosine similarity as a distance function or through a learned network.

IR also arises in contexts other than NLP, such as image and video retrieval \citep{Long_2020_CVPR}.  Image and video retrieval challenges, such as those proposed by \citet{caba2015activitynet,kaggle}, have inspired the development of a wide variety of techniques.\\


\noindent{\em Topology-enabled IR.}
Prior work that adapts topology to IR proposes topological analysis systems in the setting of point-set topology: these models are theoretical, and use separation axioms in topological spaces to describe the restriction on topologies to define a threshold for retrieval \citep{EGGHE199861, CLEMENTINI1994815, https://doi.org/10.1002/(SICI)1097-4571(1998)49:13<1144::AID-ASI2>3.0.CO;2-Z, Everett1992, https://doi.org/10.1002/(SICI)1097-4571(199212)43:10<658::AID-ASI3>3.0.CO;2-H}.  A tool set based on point-set topology has also been proposed to improve the performance of enterprise-related document retrieval \citep{enterprise-topology}.  

More recently, and relevant to the task of IR, \cite{aloni2021joint} propose a method that jointly implements geometric and topological concepts---including persistent homology---to build a representation of hierarchical databases.  Our approach is in contrast to this work and other data representation or compression tasks, such as that by \cite{Moor2020}, in the sense that we do not seek to impose geometric or topological structure on a database nor do we explicitly enforce a topological constraint.  Rather, we seek to study its preconditioned hierarchical structure using persistent homology and use this information to perform IR.



\section{Topological Data Analysis}
\label{sec:tda}

TDA adapts algebraic topology to data.  Classical topology studies features of spaces that are invariant under smooth transformations (e.g., ``stretching" without ``tearing").  A prominent example for such features is $k$-dimensional holes, where dimension 0 corresponds to connected components, dimension 1 to cycles, dimension 2 to voids, and so on.  Homology is a theoretical concept that algebraically identifies and counts these features.  The crux of topological data analysis is that topology captures meaningful aspects of the data, while algebraicity lends interpretability and computational feasibility.

\subsection{Persistent Homology}

Persistent homology adapts homology to data and outputs a set of topological descriptors that summarizes topological information of the dataset.  Data may be very generally represented as a finite point cloud, which is a collection of points sampled from an unknown manifold, together with some similarity measure or metric. Point clouds may therefore be viewed as finite metric spaces.  Persistent homology assigns a continuous, parameterized sequence of nested skeletal structures to the point cloud according to a user-specified proximity rule.  The appearance and dissipation of holes in this sequence is tracked, providing a concise summary of the presence of homological features in the data at all resolutions.  We now briefly formalize these technicalities; a complete discussion with full details can be found in the literature on applied and computational topology (e.g., \citet{carlsson2009topology, ghrist2008barcodes, edelsbrunner2008persistent}).



Although the underlying idea of homology is intuitive, computing homology can be challenging.  One convenient workaround is to study a discretization of the topological space as a union of simpler building blocks assembled in a combinatorial manner.  When the building blocks are simplices (e.g., vertices, lines, triangles, and higher-dimensional facets), the skeletonized version of the space as a union of simplices is a simplicial complex, and the resulting homology theory is simplicial homology, for which there exist efficient computational algorithms (see, e.g., \citet{munkres2018elements} for a background reference on algebraic topology).  Thus, in this paper, we will use simplicial homology over a field to study a finite topological space $X$; $(X, d_X)$ is therefore a finite metric space.


A {\em $k$-simplex} is the convex hull of $k+1$ affinely independent points $x_0, x_1, \ldots, x_k$, denoted by $[x_0,\, x_1,\, \ldots,\, x_k]$; a set of $k$-simplices forms a simplicial complex $K$.  Simplicial homology is based on simplicial $k$-chains, which are linear combinations of $k$-simplices in finite $K$ over a field $\mathbb{F}$. A set of $k$-chains thus defines a vector space $C_k(K)$.  

\begin{definition}
The {\em boundary operator} $\partial_k: C_k(K) \to C_{k-1}(K)$ maps to lower dimensions of the vector spaces by sending simplices $[x_0, x_1, \ldots, x_k] \mapsto \sum_{i=0}^k (-1)^i[x_0, \ldots, \hat{x}_{i}, \ldots, x_k]$ with linear extension, where $\hat{x}_i$ indicates that the $i$th element is dropped.  $B_k(K) := \text{im~} \partial_{k+1}$ is the set of {\em boundaries}; $Z_k(K) := \text{ker~}\partial_{k}$ is the set of {\em cycles}.  The {\em $k$th homology group} of $K$ is the quotient group $H_k(K) := Z_k(K)/B_k(K)$.
\end{definition}

Homology documents the structure of $K$, which is a finite simplicial complex representation of a topological space $X$.  Simplicial complexes $K$ may be constructed according to various assembly rules, which give rise to different complex types.  In this paper, we work with the {\em Vietoris--Rips} (VR) complexes and filtrations.

\begin{definition}
\label{def:VRcomplex}
Let $(X, d_X)$ be a finite metric space, let $r \in \mathbb{R}_{\geq 0}$.  The {\em Vietoris--Rips complex} of $X$ is the simplicial complex with vertex set $X$ where $\{x_0,\, x_1,\, \ldots,\, x_k\}$ spans a $k$-simplex if and only if the diameter $d(x_i, x_j) \leq r$ for all $0 \leq i, j \leq k$.

A {\em filtration} of a finite simplicial complex $K$ is a sequence of nested subcomplexes $K_0 \subseteq K_1 \subseteq \cdots \subseteq K_t = K$; a simplicial complex $K$ that can be constructed from a filtration is a filtered simplicial complex.  In this paper, filtrations are indexed by a continuous parameter $r \in [0, t]$; setting the filtration parameter to be the diameter $r$ of a VR complex yields the {\em Vietoris--Rips filtration}.
\end{definition}

Persistent homology computes homology in a continuous manner when the simplicial complex evolves continuously over a filtration.
\begin{definition}
Let $K$ be a filtered simplicial complex.  The {\em $k$th persistence module derived in homology} of $K$ is the collection $\text{PH}_k(K) := \{ H_k(K_r) \}_{0 \leq r \leq t}$,
together with associated linear maps $\{ \varphi_{r,s} \}_{0 \leq r < s \leq t}$ where $\varphi_{r,s}:H_k(K_r) \rightarrow H_k(K_s)$ is induced by the inclusion $K_r \hookrightarrow K_s$ for all $r,s \in [0,t]$ where $r \leq s$.
\end{definition}

Persistent homology, therefore, contains information not only on the individual spaces $\{K_r\}$ but also on the mappings between every pair $K_r$ and $K_s$ where $r \leq s$.  Persistent homology keeps track of the continuously evolving homology of $X$ across all scales as topological features (captured by simplices) appear, evolve, and disappear as the filtered simplicial complex $K$ evolves with the filtration parameter $r$.  

Persistent homology outputs a collection of intervals where each interval represents a topological feature in the filtration; the left endpoint of the interval signifies when each feature appears (or is ``born"), the right endpoint signifies when the feature disappears or merges with another feature (``dies"), and the length of the interval corresponds to the feature's ``persistence."  Each interval may be represented as a set of ordered pairs and plotted as a persistence diagram.  A persistence diagram of a persistence module, therefore, is a multiset of points in $\mathbb{R}^2$.  Typically, points located close to the diagonal in a persistence diagram are interpreted as topological ``noise" while those located further away from the diagonal can be seen as ``signal."  The distance from a point on the persistence diagram to its projection on the diagonal is a measure of the topological feature's persistence.  Figure \ref{fig:ph_ex} provides an illustrative depiction of a VR filtration and persistence diagram.  In real data applications (and in this paper), persistence diagrams are finite.

\begin{figure*}
\begin{center}
\centerline{\includegraphics[width=\columnwidth]{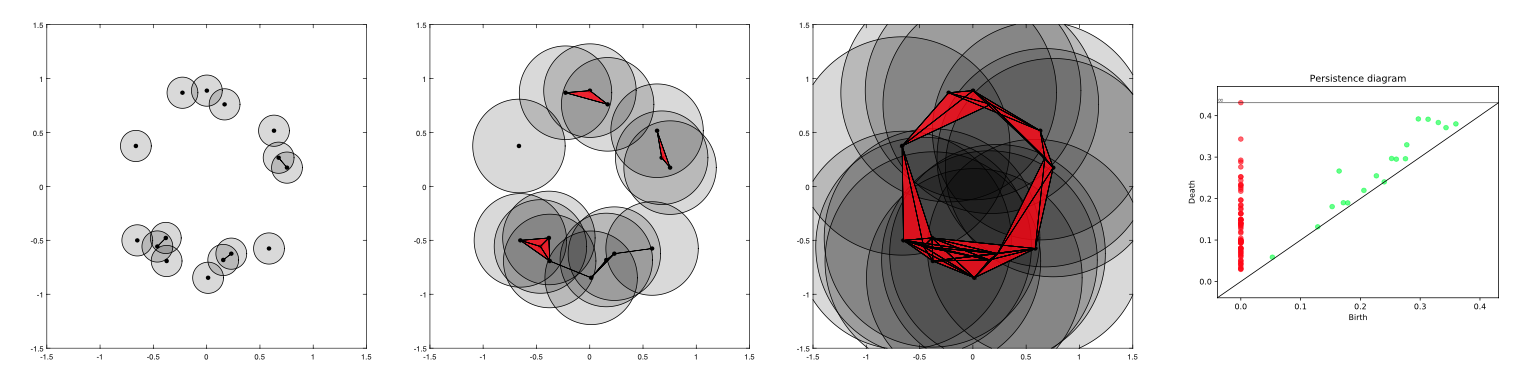}}
\caption{Illustration of a VR filtration and persistent homology.  For a sample of points, a closed metric ball is grown around each point.  The radius of the balls is the VR threshold, $r$.  Points are connected by an edge whenever two balls have a non-empty intersection; $k$-simplices are spanned for intersections of $k$ balls.  As $r$ grows, the collection of $k$-simplices evolves, yielding a VR filtration.  The topology of the VR filtration is tracked in terms of the homological features (connected components, cycles, voids) that appear (are ``born"), evolve, and disappear (``die") as $r$ grows, and documented in a persistence diagram (right-most).  In this figure, red points correspond to $H_0$ homology, or connected components, and green points correspond to $H_1$ homology, or cycles.}
\label{fig:ph_ex}
\end{center}
\end{figure*}

\subsection{Distances Between Persistence Diagrams}  
The set of all persistence diagrams can be endowed with various distance measures; under mild regularity conditions (of local finiteness of persistence diagrams), it is a rigorous metric space.  The {\em bottleneck distance} measures distances between two persistence diagrams as a minimal matching between the two diagrams, allowing points to be matched with the diagonal $\Delta$, which is the multiset of all $(x,x) \in \mathbb{R}^2$ with infinite multiplicity, or the set of all zero-length intervals.

\begin{definition}
Let $\mathcal{M}$ be the space of all finite metric spaces. Let ${\rm Dgm}(X,d_X)$ be the persistence diagram corresponding to the VR persistent homology of $(X,d_X)$.  Let $\mathcal{D}=\{{\rm Dgm}(X,d_X) \mid (X,d_X)\in\mathcal{M}\}$ be the space of persistence diagrams of VR filtrations for finite metric spaces in $\mathcal{M}$.  The {\em bottleneck distance} on $\mathcal{D}$ is given by
$$
d_{\infty}({\rm Dgm}(X,d_X),\, {\rm Dgm}(Y,d_Y))= \inf_\gamma\sup_{x\in{\rm Dgm}(X,d_X)}\|x-\gamma(x)\|_\infty,
$$
where $\gamma$ is a multi-bijective matching of points between ${\rm Dgm}(X, d_X)$ and ${\rm Dgm}(Y, d_Y)$. 
\end{definition}

The fundamental stability theorem in persistent homology asserts that small perturbations in input data will result in small perturbations in persistence diagrams, measured by the bottleneck distance \citep{cohen2007stability}.  This result together with its computational feasibility renders the bottleneck distance the canonical metric on the space of persistence diagrams for approximation studies in persistent homology.

Persistence diagrams, however, are not robust to the scaling of input data and metrics: the same point cloud measured by the same metric in different units results in different persistence diagrams with a potentially large bottleneck distance between them.  One way that this discrepancy has been previously addressed is to study filtered complexes on a log scale, such as in \citet{bobrowski2017, BUCHET201670, 10.1145/2535927}.   In these works, the persistence of a point in the persistence diagram (in terms of its distance to the diagonal) is based on the ratio of birth and death times, which alleviates the issue of artificial inflation in persistence arising from scaling.  Nevertheless, these approaches fail to recognize when two diagrams arising from the same input data measured by the same metric but in different units should coincide.  The {\em shift-invariant bottleneck distance} proposed by \cite{sheehy2018computing} addresses this issue by minimizing over all shifts along the diagonal, which provides scale invariance for persistence diagrams.  

A further challenge is to recognize when the persistent homology of an input dataset is computed with two different metrics and we would like to deduce that the topology the dataset is robust to the choice of metric.  Neither log-scale persistence nor the shift-invariant bottleneck distance is able to recognize when two persistence diagrams are computed from the same dataset but with different metrics.  We therefore introduce the dilation-invariant bottleneck dissimilarity which mitigates this effect of metric distortion on persistence diagrams.  The dilation-invariant bottleneck dissimilarity erases the effect of scaling between metrics used to construct the filtered simplicial complex and instead focuses on the topological structure of the point cloud.

\subsection{Dilation-Invariant Bottleneck Comparison Measures}



The dilation-invariant bottleneck comparison measures we propose are motivated by the {\em shift-invariant bottleneck distance} proposed by \cite{sheehy2018computing}, which we now briefly overview.  Let $\mathcal{F}=\{f:X\to \mathbb{R}\}$ be the space of continuous tame functions over a fixed triangulable topological space $X$, equipped with supremum metric $d_\mathcal{F}(f,g)=\sup_x |f(x)-g(x)|$. Define an isometric action of $\mathbb{R}$ on $\mathcal{F}$ by
\begin{align*}
	\mathbb{R}\times\mathcal{F}&\to\mathcal{F}\\
	(c,f)&\mapsto f+c.
\end{align*}
The resulting quotient space $\overline{\mathcal{F}}=\mathcal{F}/\mathbb{R}$ is a metric space with the quotient metric defined by
$$
	\overline{d_\mathcal{F}}([f],[g])=\inf_{c\in\mathbb{R}}\sup_{x\in X}|f(x)+c-g(x)|=\inf_{c\in\mathbb{R}}d_\mathcal{F}(f+c,g).
$$

Let $\mathcal{D}_\infty=\{{\rm Dgm}(f) \mid f\in\mathcal{F}\}$ be the space of persistence diagrams of functions in $\mathcal{F}$, equipped with the standard bottleneck distance. Similarly, we can define the action  
\begin{align*}
	\mathbb{R}\times\mathcal{D}_\infty&\to\mathcal{D}_\infty\\
	(c,{\rm Dgm}(f))&\mapsto {\rm Dgm}(f+c).
\end{align*}  
The resulting quotient space $\overline{\mathcal{D}_\infty}=\mathcal{D}/\mathbb{R}$ is a metric space with the quotient metric defined by	
	\begin{equation}
	\label{eq:SI_distance}
	\overline{d_{S}}([{\rm Dgm}(f)],\, [{\rm Dgm}(g)])=\inf_{c\in\mathbb{R}}d_{\infty}({\rm Dgm}(f+c),\, {\rm Dgm}(g)).
	\end{equation}
The quotient metric $\overline{d_S}$ is precisely the shift-invariant bottleneck distance: effectively, it ignores translations of persistence diagrams, meaning that two persistence diagrams computed from the same metric but measured in different units will measure a distance of zero with the shift-invariant bottleneck distance.

From this construction, we derive the dilation-invariant dissimilarity and distance in the following two ways.  In a first instance, we replace shifts of tame functions by dilations of finite metric spaces.  A critical difference between our setting and that of the shift-invariant bottleneck distance is that dilation is not an isometric action, thus the quotient space can only be equipped with a dissimilarity.  

In our second derivation, we convert dilations to shifts using the log map which gives us a proper distance function, however is valid only for positive elements, i.e., diagrams with no persistence points born at time 0.


Before we formalize the dilation-invariant bottleneck dissimilarity and distance, we first recall the Gromov--Hausdorff distance \citep[e.g.,][]{burago2001course,degregorio2020notion}.

\begin{definition}[Gromov--Hausdorff Distance]
\label{def:GH}
For two metric spaces $(X, d_X)$ and $(Y, d_Y)$, a {\em correspondence} between them is a set $R \subseteq X \times Y$ where $\pi_X(R) = X$ and $\pi_Y(R) = Y$ where $\pi_X$ and $\pi_Y$ are canonical projections of the product space; let $\mathcal{R}(X,Y)$ denote the set of all correspondences between $X$ and $Y$.  A {\em distortion} $\tilde{d}$ of a correspondence with respect to $d_X$ and $d_Y$ is
\begin{equation}
\label{eq:distortion}
\tilde{d}(R, d_X, d_Y) = \sup_{(x,y), (x',y') \in R} |d_X(x,x') - d_Y(y,y')|.
\end{equation}
The {\em Gromov--Hausdorff} distance between $(X, d_X)$ and $(Y, d_Y)$ is
\begin{equation}
\label{eq:GH}
d_{\mathrm{GH}}((X, d_X),\, (Y, d_Y)) = \frac{1}{2} \inf_{R \in \mathcal{R}(X,Y)} \tilde{d}(R, d_X, d_Y).
\end{equation}
\end{definition}

Intuitively speaking, the Gromov--Hausdorff distance measures how far two metric spaces are from being isometric.

\begin{remark}
In the case of finite metric spaces, such as in this paper, the supremum in \eqref{eq:distortion} is in fact a maximum and the infimum in \eqref{eq:GH} is a minimum.
\end{remark}

\subsubsection{The Dilation-Invariant Bottleneck Dissimilarity}

Let $(\mathcal{M},d_{\text{GH}})$ be the space of finite metric spaces, equipped with Gromov--Hausdorff distance $d_{\text{GH}}$.  Let $\mathbb{R}_+$ be the multiplicative group on positive real numbers.  Define the group action of $\mathbb{R}_+$ on $\mathcal{M}$ by 
\begin{equation}
\label{eq:GH}
\begin{aligned}
\mathbb{R}_+\times\mathcal{M}&\to\mathcal{M},\\
(c,(X,d_X))&\mapsto (X,c\cdot d_X).
\end{aligned}
\end{equation}  
Unlike the shift-invariant case, here, the group action is not isometric and thus we cannot obtain a quotient metric space.  Instead, we have an asymmetric dissimilarity defined by 
\begin{equation*}
\overline{d_\text{GH}}((X,d_X),\, (Y,d_Y))=\inf_{c\in\mathbb{R}_+}d_\text{GH}((X,c\cdot d_X),\, (Y,d_Y)).
\end{equation*}
Note that multiplying a constant to the distance function results in a dilation on the persistence diagram and that for VR persistence diagrams ${\rm Dgm}(X,d_X)$, a simplex is in the VR complex of $(X,d_X)$ with threshold $r$ if and only if it is in the VR complex of $(X,c\cdot d_X)$ with threshold $cr$.  Therefore, any point in the persistence diagram ${\rm Dgm}(X,c\cdot d_X)$ is of the form $(cu,cv)$ for some $(u,v)\in{\rm Dgm}(X,d_X)$.  This then gives the following definition.

\begin{definition}
\label{def:DI-dissimilarity}
For two persistence diagrams $\mathrm{Dgm}(X, d_X), \mathrm{Dgm}(Y, d_Y)$ corresponding to the VR persistent homology of finite metric spaces $(X, d_X)$ and $(Y, d_Y)$, respectively, and a constant $c$ in the multiplicative group on positive real numbers $\R_+$, the {\em dilation-invariant bottleneck dissimilarity} is given by 
\begin{equation}\label{eq:DI-dissimilarity}
\overline{d_{D}}({\rm Dgm}(X,d_X),\, {\rm Dgm}(Y,d_Y))=
\inf_{c\in\mathbb{R}_+}d_{\infty}({\rm Dgm}(X,c \cdot d_X),\, {\rm Dgm}(Y,d_Y)).
\end{equation}
The {\em optimal dilation} is the positive number $c^* \geq 0$ such that $\overline{d_D}({\rm Dgm}(X,d_X),\, {\rm Dgm}(Y,d_Y)) = d_\infty(c^* \cdot {\rm Dgm}(X,d_X),\, {\rm Dgm}(Y,d_Y))$.
\end{definition}

As we will see later on in Section \ref{sec:methods} when we discuss computational aspects, the optimal dilation is always achievable and the infimum in \eqref{eq:DI-dissimilarity} can be replaced by the minimum. This is possible if we assume the value of $c=0$ falls within the search domain, which would amount to $Y$ being a single-point space, which will be further discussed in this section. Assuming this replacement of the minimum, we now summarize properties of $\overline{d_{D}}$.

\begin{proposition}\label{prop:DI-dissimilarity}
Let $A$ and $B$ be two persistence diagrams.  The dilation-invariant bottleneck dissimilarity (\ref{eq:DI-dissimilarity}) satisfies the following properties.
\begin{enumerate}
\item {\em Positivity:} $\overline{d_{D}}(A,B)\ge 0$, with equality if and only if $B$ is proportional to $A$, i.e., $B=c^*A$ for the optimal dilation $c^*$.  The dissimilarity correctly identifies persistence diagrams under the action of dilation.

\item {\em Asymmetry:} The following inequality holds:
\begin{equation}\label{eq:asymmetry}
\overline{d_{D}}(A, B)\ge c^* \cdot \overline{d_{D}}(B,A).
\end{equation}

\item {\em Dilation Invariance:} For any $c>0$, $\overline{d_{D}}(cA,B)=\overline{d_{D}}(A,B)$. For a dilation on the second variable, 
    	\begin{equation}\label{eq:second-variable}
    	 \overline{d_{D}}(A,cB)=c \cdot \overline{d_{D}}(A,B).
    	\end{equation}

\item {\em Boundedness:} Let $D_0$ be the empty diagram; i.e., the persistence diagram with only the diagonal of points with infinite multiplicity $\Delta$, and no persistence points away from $\Delta$. Then
    \begin{equation}
    \overline{d_{D}}(A,B)\le \min\{d_{\infty}(A,B),\, d_{\infty}(D_0, B)\}.
    \end{equation}
\end{enumerate}
\end{proposition}

\begin{proof}
To show positivity, the ``only if'' direction holds due to the fact that the optimal dilation is achievable. Thus, $\overline{d_D}(A, B)=d_\infty(c^*A,B)=0$ for $c^* \geq 0$ and $B=c^*A$ since $d_\infty$ is a distance function.

To show asymmetry, we consider the bijective matching $\gamma:c^* A\to B$. If $c^*=0$, the inequality holds by positivity. If $c^*>0$, the inverse $\gamma^{-1}$ yields a matching from $\displaystyle \frac{1}{c^*}B\to A$ with matching value $\displaystyle \frac{1}{c^*} \cdot \overline{d_{D}}(A, B)$. Therefore $\displaystyle c^*\overline{d_{D}}(B, A)\le \overline{d_{D}}(A,B)$.

To show dilation invariance, notice that by Definition \ref{def:DI-dissimilarity}, $\overline{d_D}$ will absorb any coefficient on  the first variable. It therefore suffices to prove that \eqref{eq:second-variable} holds on the second variable. Let $c^*$ be the dilation such that $\overline{d_{D}}(A, B)=d_{\infty}(c^* A,B)$, and let $\gamma:c^* A\to B$ be the optimal matching. Then for any other $c>0$, $ca\mapsto c\gamma(a)$ is also a bijective matching from $c\cdot c^* A\to cB$. This gives $\overline{d_{D}}(A, cB)\le c\cdot \overline{d_{D}}(A,B)$. Conversely, if there is a dilation $c'$ and a bijective matching $\gamma':c' A\to cB$ such that the matching value is less than $c\cdot \overline{d_{D}}(A, B)$, then $\displaystyle \frac{c'}{c}a\mapsto \frac{1}{c}\gamma'(c'a)$ gives a bijective matching from $\displaystyle \frac{c'}{c} A\to B$ with matching value less than $\overline{d_{D}}(A, B)$ which contradicts the selection of $c^*$. Therefore, $\overline{d_{D}}(A, cB)=c\cdot \overline{d_{D}}(A, B)$.
    
Finally, to show boundedness, note that the function $\Theta(c)=d_\infty(cA,B)$ is continuous on $\mathbb{R}_+$, if we take a sequence of dilations approaching 0, then $\overline{d_{D}}(A, B)\le d_{\infty}(B,D_0)$.  If $c=1$, then $\overline{d_{D}}(A, B)=d_\infty(A,B)$.  Therefore, $\overline{d_{D}}(A, B)\le \min\{d_\infty(A,B),\, d_\infty(D_0, B)\}$.
\end{proof}

\begin{figure}[htbp]
	\centering
	\subfigure[]{
	\includegraphics[width=0.3\linewidth]{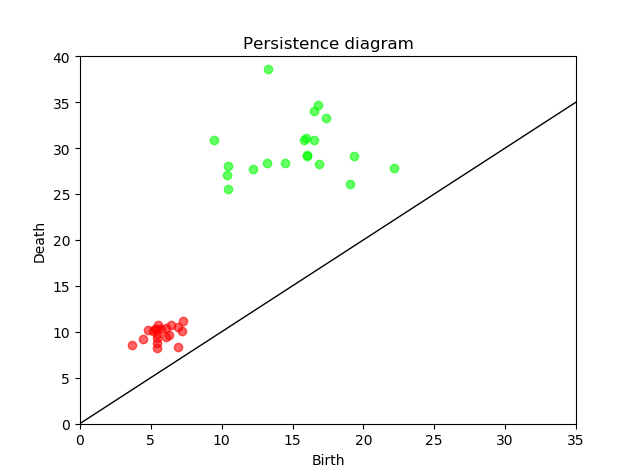}}
	\subfigure[]{
	\includegraphics[width=0.3\linewidth]{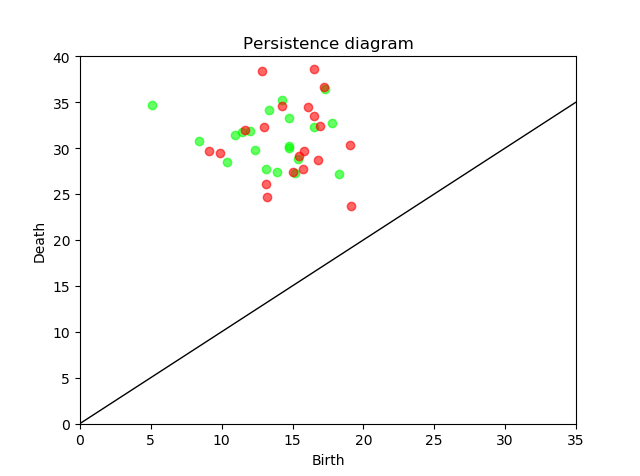}}
	\subfigure[]{
	\includegraphics[width=0.3\linewidth]{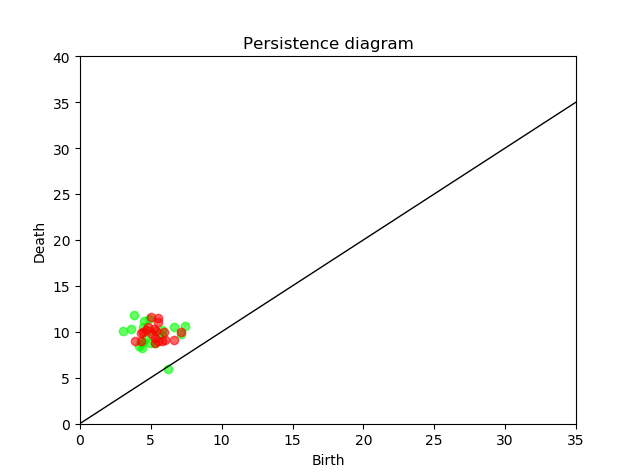}}
	\caption{A synthetic example to illustrate the asymmetry of $\overline{d_{D}}$: (a) Two synthetic persistence diagrams: red and green; (b) Scaling the green diagram to obtain $\overline{d_{D}}(\text{green},\, \text{red})$; (c) Scaling the red diagram to obtain $\overline{d_{D}}(\text{red},\, \text{green})$. The value of $\overline{d_D}$ depends on the scale of the second variable: as seen here, $\overline{d_D}(\text{green},\text{red})$ is much smaller than $\overline{d_D}(\text{red},\text{green})$.}  
	\label{fig:asymmetry}
\end{figure}

\begin{remark}
A simple illustration of the asymmetry of $\overline{d_D}$ is the following: let $A=D_0$ be the empty persistence diagram and $B$ be any nonempty diagram. Then $\overline{d_D}(A, B)= d_\infty(D_0, B)$ but $\overline{d_D}(B, A)=0$. In this example, equality of \eqref{eq:asymmetry} cannot be achieved, so it is a strict inequality.

We also give an illustrative example of asymmetry of $\overline{d_D}$ in Figure \ref{fig:asymmetry}.
\end{remark}

Finally, as a direct corollary of the fundamental stability theorem in persistent homology \citep{cohen2007stability,https://doi.org/10.1111/j.1467-8659.2009.01516.x}, we have the following stability result for the dilation-invariant bottleneck dissimilarity.
\begin{proposition}[Dilation-Invariant Stability]
Given two finite metric spaces $(X,d_X)$ and $(Y,d_Y)$, we have that 
\begin{equation}
\label{eq:dissimilarity-stability}
			\overline{d_{D}}( \mathrm{Dgm}(X,d_X),\,  \mathrm{Dgm}(Y,d_Y) )\le 2\overline{d_{\mathrm{GH}}}((X,d_X),\, (Y,d_Y))
\end{equation}
where $\mathrm{Dgm}$ denotes the persistence diagram of the VR filtration of the two metric spaces.
\end{proposition} 
\begin{proof}
	In \cite{https://doi.org/10.1111/j.1467-8659.2009.01516.x} the stability theorem states that for any $c>0$,
		\begin{equation}\label{eq:gh-stability}
			d_\infty(\mathrm{Dgm}(X,cd_X),\, \mathrm{Dgm}(Y,d_Y))\le 2d_{\mathrm{GH}}((X,cd_X),(Y,d_Y))
		\end{equation}
	Taking the infimum of both sides of \eqref{eq:gh-stability} gives the result, \eqref{eq:dissimilarity-stability}. 
\end{proof}

\subsubsection{The Dilation-Invariant Bottleneck Distance}

We now study a construction that gives rise to a distance function.  Here, we consider persistence diagrams where all points have positive coordinates. 

\begin{definition}
Consider the equivalence classes defined by dilation given as follows: for persistence diagrams $A$ and $A'$, we set $A\sim A'$ if and only if $A=cA'$ for some positive constant $c$. Let $[A]$ and $[B]$ be two equivalence classes. The {\em dilation-invariant bottleneck distance} is given by 
\begin{equation}
\label{eq:DI_distance}
	D_S([A],[B])=\overline{d_S}(\log(A),\, \log(B)),
\end{equation}
where $\overline{d_S}$ is the quotient metric \eqref{eq:SI_distance}.
\end{definition}

\begin{proposition}
	The function $D_S$ given by \eqref{eq:DI_distance} is a well-defined metric on the equivalence classes defined by dilation.
\end{proposition}
\begin{proof}
	If $A'=cA$ is another representative in the class $[A]$, then $\log(A')=\log(A)+\log(c)$ and $\overline{d_S}(\log(A)+\log(c),\log(B))=\overline{d_S}(\log(A),\log(B))$. By symmetry it also holds for $[B]$, so $D_S$ is well-defined.
	
    We always have that $D_S([A],[B])\ge 0$ and when $\log(A)$ and $\log(B)$ differ by a shift, equality is achieved: this is equivalent to $A$ and $B$ differing by a dilation. Symmetry and the triangle inequality are consequences of $\overline{d_S}$ being a proper metric. Thus $D_S$ is a well-defined distance function, as desired.
\end{proof}

Both dilation-invariant bottleneck comparative measures are useful tools in machine learning tasks; the choice, however, of whether to use the dilation-invariant bottleneck distance or dissimilarity depends on the specific task at hand.  The dilation-invariant bottleneck distance may be more appropriate when a distance matrix is required to perform tasks such as dimension reduction, clustering, or classification. In practice, however, points with small birth/death time on persistence diagrams will need to be cropped to avoid large negative numbers tending towards infinity due to the log map. The dilation-invariant bottleneck dissimilarity, although not a distance, is more convenient to compute which is especially relevant for IR tasks, when different queries to the same database need to be compared.  Here, the dissimilarity function is more appropriate as it compares all data on the same scale.

\subsection{A Connection to Weak Isometry}
\label{sec:weak_isom}

Recent work by \cite{degregorio2020notion} defines a general equivalence relation of weak isometry as follows.  Let $(X,d_X)$ and $(Y,d_Y)$ be two finite metric spaces; they are said to be {\em weakly isometric}, $(X,d_X)=^{w}(Y,d_Y)$, if and only if there exists a bijection $\phi:X\to Y$ and a strictly increasing function $\psi:\mathbb{R}_+\to\mathbb{R}_+$ such that
$$
		\psi(d_X(x_1,x_2)) = d_Y(\phi(x_1),\, \phi(x_2)).
$$
Weak isometry thus considers all possible rescaling functions $\psi$, which includes dilations, given by 
$$
	\begin{aligned}
			\psi_c:\mathbb{R}_+&\to\mathbb{R}_+,\\
			t&\mapsto ct.
	\end{aligned}
$$
	 
For two finite metric spaces, the following function is a dissimilarity \citep{degregorio2020notion}:
	 \begin{equation}\label{eq:symmetry}
	 \begin{aligned}
	 	 	d_w((X,d_X),(Y,d_Y)) = &{} \inf_{\psi\in\mathcal{I}}d_{\mathrm{GH}}((X,\psi\circ d_X),\, (Y,d_Y)) +\inf_{\psi\in\mathcal{I}}d_{\mathrm{GH}}((X,d_X),\, (Y,\psi\circ d_Y)),
	 \end{aligned}
	 \end{equation}
where $\mathcal{I}$ is the set of all strictly increasing functions over $\mathbb{R}_+$. By definition $\tilde{d}$ is symmetric, which is critical in the proof of the following result \citep[Proposition 3]{degregorio2020notion}: 
	 \begin{equation}\label{eq:prop_symmetry}
	 	d_w((X,d_X),(Y,d_Y))=0\iff (X,d_X)=^{w}(Y,d_Y).
	 \end{equation} 
In particular, let $\psi_X=\psi\circ d_X$ and $\psi_Y=\psi\circ d_Y$. In the construction of $\psi_X$, the inverse of $\psi_Y$ is needed to show that $\psi_X$ is invertible and thus to extend $\psi_X$ to the whole of $\mathbb{R}_+$. If either of two terms in the sum of $d_w$ given by \eqref{eq:symmetry} is dropped, the relation \eqref{eq:prop_symmetry} does not hold.	
	 	
Notice that a straightforward symmetrization of the dilation-invariant bottleneck dissimilarity may be obtained by
\begin{equation}
\label{eq:DI-symmetric}
		\overline{{d}_{\mathrm{sym}}}(A,B)=\frac{\overline{d_D}(A,B)+\overline{d_D}(B,A)}{2}
\end{equation}
This symmetrization now resembles the dissimilarity \eqref{eq:symmetry} proposed by \cite{degregorio2020notion}.  However, a symmetric dissimilarity may not be practical or relevant in IR: by the dilation invariance property in Proposition \ref{prop:DI-dissimilarity}, we may end up with different scales between the query and database, leading to an uninterpretable conclusion.  An illustration of this effect is given later on with real data in Section \ref{sec:sym}.

We therefore only need one component in \eqref{eq:symmetry}, which then gives a form parallel to the dilation-invariant bottleneck dissimilarity \eqref{eq:DI-dissimilarity}. Fortunately, the conclusion \eqref{eq:prop_symmetry} still holds for the dilation-invariant bottleneck dissimilarity, even if we lose symmetry.
 
 \begin{proposition}
 		Given two finite metric spaces $(X,d_X)$ and $(Y,d_Y)$
$$
 			\overline{d_{\mathrm{GH}}}((X,d_X),\, (Y,d_Y)) = 0 
$$
 		if and only if there exists $c^*>0$ such that $(X,c^*d_X)$ is isometric to $(Y,d_Y)$, or $c^*=0$ and $(Y,d_Y)$ is a one-point space.
 \end{proposition}
 
 \begin{proof}
 It suffices to prove the necessity. Our proof follows in line with that of Proposition 3 in \cite{degregorio2020notion}; the main difference is that we do not need to prove the bijectivity and monotonicity for the rescaling function since we have already identified them as positive real numbers.
 
 Set $\overline{d_{\mathrm{GH}}}=0$.  There exists a sequence of real positive numbers $\{c_n\}$ such that 
 \begin{equation}\label{eq:lim_gh-distance}
     \lim_{n\to \infty}d_{\mathrm{GH}}((X,c_n d_X),\, (Y,d_Y))=0.
 \end{equation}
 Recall from Definition \ref{def:GH}, the Gromov--Hausdorff distance is given by a distortion of correspondences. Thus, condition \eqref{eq:lim_gh-distance} is equivalent to saying there is a sequence $\{c_n\}$ and a sequence of correspondences $\{R_n\}$ such that
$
 \lim_{n\to\infty}\tilde{d}(R_n,\, (X,c_nd_X),\, (Y,d_Y))=0.
$
 For finite metric spaces, the set of correspondences $\mathcal{R}$ is also finite. By passing to subsequences, we can find a correspondence $R_0$ and a subsequence $\{c_{n_k}\}$  such that
$
 			\lim_{k\to\infty}|c_{n_k}d_X(x,x')-d_Y(y,y')|=0
$
for all $(x,y), (x',y') \in R_0$.
 Thus the sequence $\{c_{n_k}\}$ converges to a number $c^*\ge 0$ such that
$
 d_{\mathrm{GH}}((X,c^*d_X),\, (Y,d_Y))=0.
$
 
 If $c^*>0$ then $(X,c^*d_X)$ is isometric to $(Y,d_Y)$, as desired; if $c^*=0$, then $(X, 0\cdot d_X)$ degenerates to a one-point space (since all points are identified) and so does $(Y, d_Y)$.
 \end{proof}



\section{Implementation of Methods}
\label{sec:methods}

We now give computation procedures for the dilation-invariant bottleneck dissimilarity and distance.  We also give theoretical results on the correctness and complexity of our algorithms.  Finally, we give a performance comparison of our proposed algorithm against the current algorithm to compute shift-invariant bottleneck distances \cite{sheehy2018computing}.

\subsection{Computing the Dilation-Invariant Bottleneck Dissimilarity}

We now give an algorithm to search for the optimal dilation and compute the dilation-invariant bottleneck dissimilarity. The steps that we take are as follows: first, we show that the optimal dilation value $c^*$ falls within a compact interval, which is a consequence of the boundedness property in Proposition \ref{prop:DI-dissimilarity}. Then we present our algorithm to search for the optimal dilation $c^*$ and assess the correctness of the algorithm and analyze its computational complexity. Finally, we improve the search interval by tightening it to make our algorithm more efficient, and reveal the structure of the dissimilarity function with respect to the dilation parameter. 

For convenience, we first state the following definitions.

\begin{definition}
Let $A$ be a persistence diagram.
\begin{enumerate}
    \item The {\em persistence of a point} $a=(a_x,a_y)\in\mathbb{R}^2$ is defined as $\pers(a):=\frac{a_y-a_x}{2}$. For a persistence diagram $A$, the {\em characteristic diagram} $\chi(A)$ is the multiset of points in $A$ with the largest persistence; 
    \item The {\em persistence of a diagram} is defined as the largest persistence of its points $\pers(A):=\max_{a\in A}\{\pers(a)\}=\pers(\chi(A))$;
    \item The {\em bound of a diagram} is defined as the largest $\infty$-norm of its points $\bd(A):=\max_{a\in A}\{|a|_\infty\}=\max_{a\in A}\{a_y\}$;
\end{enumerate}
\end{definition}
As above in Proposition \ref{prop:DI-dissimilarity}, the empty diagram $D_0$ consists of only the diagonal of points with infinite multiplicity $\Delta$ and we adopt the convention that $D_0 = 0A$ for every $A$.

\subsubsection{Existence of an Optimal Dilation $c^*$}


Searching for $c^*$ over the entire ray $(0,\infty)$ is intractable, but thanks to the boundedness of $\overline{d_D}(A, B)$, we only need to search for $c^*$ over values of $c$ where $cA$ lies in the ball centered at $B$ with radius $d_0$, where $d_0$ is the smaller of the standard bottleneck distance either between $A$ and $B$ or between $B$ and the empty diagram $D_0$. This yields a compact interval within which $c^*$ lives.
  


 \begin{theorem}
  	Let $A$ and $B$ be persistence diagrams and set $d_0 := \min\{d_{\infty}(A,B),\, d_{\infty}(D_0, B)\}$. The optimal dilation $c^*$ lies in the interval
$$  	
\bigg[\frac{\pers(B)-d_0}{\pers(A)},\frac{\pers(B)+d_0}{\pers(A)} \bigg].
$$
  \end{theorem} 

\begin{proof}
For values of $\displaystyle c>\frac{\pers(B)+d_0}{\pers(A)}$, we have 
	\begin{align*}
	d_\infty(cA,B) \ge d_\infty(cA,\Delta)-d_\infty(B,\Delta) & = c\pers(A)-\pers(B)\\
	& > \pers(A)\frac{\pers(B)+d_0}{\pers(A)}-\pers(B)
 = d_0\ge \overline{d_D}(A,B).
	\end{align*}
	Similarly, for values of $\displaystyle c<\frac{\pers(B)-d_0}{\pers(A)}$, then
$
		\pers(cA)=c\cdot \pers(A)<\pers(B)-d_0.
$
	Therefore,
$$
		d_\infty(cA,B) \ge d_\infty(B,D_0)-d_\infty(cA,D_0)>d_0\ge \overline{d_D}(A,B).
$$
	These observations combined give us the result that the optimal dilation must reside in the interval
$$
		\bigg[\frac{\pers(B)-d_0}{\pers(A)},\frac{\pers(B)+d_0}{\pers(A)} \bigg].
$$
\end{proof}

\begin{remark}
The optimal dilation $c^*$ need not be unique.  In numerical experiments, we found that there can be multiple occurrences of $c^*$ within the search interval.
\end{remark}

\subsubsection{A Grid Search Method}\label{sec:direct-searching}

Set $\displaystyle c_{\min}:=\frac{\pers(B)-d_0}{\pers(A)}$ and $\displaystyle c_{\max}:=\frac{\pers(B)+d_0}{\pers(A)}$. Consider the uniform partition
$
	c_{\min}=t_0<t_1<\cdots<t_N=c_{\max}
$
where
$$
	t_{i+1}-t_i = \frac{c_{\max}-c_{\min}}{N} = \frac{2d_0}{N\pers(A)}.
$$
We then search for the minimum of the following sequence
\begin{equation}
\label{eq:dissim_est}
	\widehat{d_D}(A,B):=\min\{d_\infty(t_iA,B)\}_{i=0}^N
\end{equation}
and set $c^* := \mathop{\arg\min}_{t_i}\{d_\infty(t_iA,B)\}_{i=0}^N$.
\paragraph{Complexity:} Our grid search approach calls any bipartite matching algorithm to compute bottleneck distances $N$ times. In software packages such as \texttt{persim} in \texttt{scikit-tda} \citep{scikittda2019}, the Hopcroft--Karp algorithm is used to compute the standard bottleneck distance. The time complexity of the Hopcroft--Karp algorithm is $O(n^{2.5})$, where $n$ is the number of points in the persistence diagrams \citep{hk1973}.  Thus the overall complexity for our direct search algorithm is $O(Nn^{2.5})$. \citet{efrat2001} present a bipartite matching algorithm for planar points with time complexity $O(n^{1.5}\log n)$ which was then implemented to the setting of  persistence diagrams by \citet{morozov2017}. The total complexity of our direct search can therefore be improved to $O(Nn^{1.5}\log n)$ using software libraries such as \texttt{GUDHI} \citep{gudhi:urm} or \texttt{Hera} \citep{kerber52hera}. 

\paragraph{Correctness:} The convergence of this algorithm relies on the continuity of the standard bottleneck distance function. Specifically, for $\widehat{d_D}(A,B)$ given by \eqref{eq:dissim_est}, we have the following convergence result.
\begin{theorem}
\label{thm:correct}
    Given any partition $c_{\min}=t_0<t_1<\cdots<t_N=c_{\max}$,
$$
		0\le \widehat{d_D}(A,B)-\overline{d_D}(A,B)\le \frac{2d_0\bd(A)}{N\pers(A)}.
$$
\end{theorem}
\begin{proof}
	For any $c,c'>0$, by the triangle inequality
$
		|d_\infty(cA,B)-d_\infty(c'A,B)|\le d_\infty(cA,c'A).
$
	Consider the trivial bijection given by $ca\mapsto c'a$ for any $a\in A$. Thus, 
$$
		d_\infty(cA,c'A)\le |c-c'|\cdot \max_{a\in A}\|a\|_\infty = |c-c'|\bd(A).
$$
	For any $N$-partition of the interval $[c_{\min},c_{\max}]$, the optimal dilation must lie in a subinterval $c^* \in[t_j, t_{j+1}]$ for some $j$. Then
$$
			\widehat{d_D}(A,B)-\overline{d_D}(A,B)\le d_\infty(t_jA,B)-\overline{d_D}(A,B)
			\le d_\infty(t_jA,c^* A)
			\le \frac{2d_0}{N\pers(A)}\bd(A).
$$
\end{proof}

\begin{remark}
The convergence rate is at least linear with respect to partition number $N$.    
\end{remark}


In light of Theorem \ref{thm:correct}, notice that although the standard bottleneck distance function is continuous, it may not be differentiable with respect to dilation. 

\begin{remark}
Recent results generalize persistence diagrams to {\em persistence measures} and correspondingly, the bottleneck distance to the {\em optimal partial transport distance} \citep{divol2021understanding}. In this setting, the optimization problem can be relaxed on the space of persistence measures \citep{lacombe2018large}, and gradient methods may be used to find the optimal dilation.
\end{remark}


\subsubsection{Tightening the Search Interval}
\label{sec:tight}

\begin{figure}[htbp]
	\centering
	\includegraphics[width=0.65\textwidth]{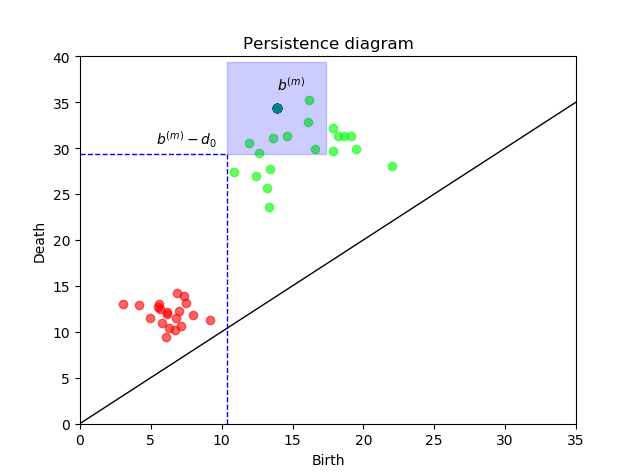}
	\caption{Illustration of the searching scheme: while we are dynamically scaling the red persistence diagram, the effective scaling must move some points in the red diagram into the blue box centered at $b^{(m)}$ with radius $d_0$. Otherwise, it is impossible to obtain a bottleneck distance less than $d_0$ which is an upper bound of the dilation-invariant bottleneck dissimilarity.}
\end{figure}

We define the function $\Theta(c) := d_\infty(cA,B)$. The optimal dilation $c^*$ is the minimizer of $\Theta(c)$ and lies in the interval $[c_{\min},c_{\max}]$. The widest interval is given by the most pessimistic scenario in terms of persistence point matching: this happens when $\pers(A)\ll\pers(B)$ and $d_\infty(A,B) = \pers(B)$, which means the bijective matching of points is degenerate and simply sends every point to the diagonal $\Delta$. In this case, $c_{\min}=0$ and $c_{\max}=\pers(B)/\pers(A)$.  This pessimistic scenario indeed arises in our IR problem of interest, as will be discussed and exemplified further on in Section \ref{sec:application}.  We can derive a sharper bound as follows.


Let $b^{(m)}\in\chi(B)$ be the point with the longest (finite) death time. Then $b^{(m)}$ has two properties:
\begin{enumerate}[(i)]
\item \ it has the largest persistence, and 
\item \ it has the largest $y$-coordinate.
\end{enumerate}
The first property (i) implies that while we are scaling $A$, we need to find a nontrivial matching point in $A$ in order to decrease the standard bottleneck distance. Thus there is at least one point lying in the box centered at $b^{(m)}$ with radius $d_0$.  The second property (ii) guarantees we can move $A$ as far as the maximal death value in order to derive a sharper bound. 


\begin{theorem}\label{thm:left-end}
	Let $b^{(m)}\in \chi(B)$ be the point with largest $y$-coordinate. If $\displaystyle c\le\frac{b_y^{(m)}-d_0}{\bd(A)}$, then $\Theta(c)\ge d_0$. Furthermore, if $\displaystyle c\le\min\bigg\{\frac{b_y^{(m)}+b_x^{(m)}}{2\bd(A)},\frac{\pers(B)}{\pers(A)}\bigg\}$, then $\displaystyle \Theta(c)=\pers(B)$.
\end{theorem}
\begin{proof}
    
	The property of largest persistence of $b^{(m)}$ implies that no point in $cA$ resides in the box around $b^{(m)}$ with radius $d_0$. In particular, for any point $a=(a_x,a_y)\in A$, note that $a_x\le a_y\le \bd(A)$ and
$$
	\|ca-b^{(m)}\|_\infty = \max\{|ca_x-b^{(m)}_x|,\, |ca_y-b^{(m)}_y|\} \ge b^{(m)}_y-ca_y > b^{(m)}_y-\frac{a_y}{\bd(A)}(b^{(m)}_y-d_0) \ge d_0.
$$
For any bijection $\gamma:B\to cA$, we have 
$
		\|b^{(m)}-\gamma(b^{(m)})\|_\infty\ge d_0
$
so $\Theta(c)\ge d_0$.
	
	Similarly, if $\displaystyle c<\frac{b_y^{(m)}+b_x^{(m)}}{2\bd(A)}$, for any point $a=(a_x,a_y)\in A$, 
$$
		\|ca-b^{(m)}\|_\infty = \max\{|ca_x-b^{(m)}_x|,\, |ca_y-b^{(m)}_y|\}
		\ge b^{(m)}_y-ca_y
		 >b^{(m)}_y-\frac{a_y}{2\bd(A)}(b^{(m)}_y+b^{(m)}_x)
		 \ge \pers(B).
$$
	Thus for any bijection $\gamma:B\to A$, we have 
$
\displaystyle		\|b^{(m)}-\gamma(b^{(m)})\|_\infty\ge \pers(B).
$
	However, the degenerate bijection sending every point to the diagonal gives $\displaystyle d_\infty(cA,B)\le\max\big\{\pers(cA),\, \pers(B)\big\}\le \pers(B)$. Thus, we have $\displaystyle \Theta(c)=\pers(B)$.
\end{proof}

This result allows us to update the lower bound of the search interval to $c_{\min} = \displaystyle \max\bigg\{\frac{b^{(m)}_y-d_0}{\bd(A)},\, \frac{\pers(B)-d_0}{\pers(A)}\bigg\}$.

 Symmetrically, by switching the roles of $A$ and $B$, we may update the upper bound of the search interval.
 
 \begin{theorem}\label{thm:right-end}
	Let $a^{(m)}\in \chi(A)$ be the point with largest $y$-coordinate. If $\displaystyle c\ge \frac{\bd(B)+d_0}{a^{(m)}_y}$, then $\Theta(c)\ge d_0$. Further, if $\displaystyle c\ge \max\bigg\{\frac{2\bd(B)}{a^{(m)}_x+a^{(m)}_y},\, \frac{\pers(B)}{\pers(A)}\bigg\}$, then $\displaystyle \Theta(c)=c\cdot \pers(A)$.
\end{theorem}
\begin{proof}
	The property of largest persistence of $b^{(m)}$ implies that no point in $B$ resides in the box around $ca^{(m)}$ with radius $d_0$. The calculation proceeds as above in the proof of Theorem \ref{thm:left-end}. With regard to the second property of $b^{(m)}$ having the largest $y$-coordinate, if $\displaystyle c\ge\frac{2\bd(B)}{a^{(m)}_x+a^{(m)}_y}$, then for any point $b=(b_x,b_y)\in B$,
\begin{align*}
		\|ca^{(m)}-b\|_\infty = \max\{|ca_x^{(m)}-b_x|,\, |ca_y^{(m)}-b_y|\}
		&\ge ca^{(m)}_y-b_y\\
		& >c\bigg(\frac{a^{(m)}_x+a^{(m)}_y}{2}+\frac{a^{(m)}_y-a^{(m)}_x}{2}\bigg)-\bd(B)\\
		& \ge c\pers(A)=\pers(cA),
\end{align*}
	thus $\displaystyle d_\infty(cA,B)\ge\pers(cA)$. The degenerate bijection matching every persistence point to the diagonal gives $\displaystyle d_\infty(cA,B)\le\max\big\{\pers(cA),\pers(B)\big\}\le \pers(cA)$. Therefore, $\Theta$ is a linear function with slope $\displaystyle \pers(A)$.
\end{proof}

The upper bound of the search interval can now be updated to $\displaystyle c_{\max}=\min\bigg\{\frac{\bd(B)+d_0}{a^{(m)}_y},\, \frac{\pers(B)+d_0}{\pers(A)}\bigg\}$.

As a consequence of Theorems \ref{thm:left-end} and \ref{thm:right-end}, we obtain the following characterization for the function $\Theta$, which is independent of the data. 
\begin{proposition}
Set $\displaystyle c_B = \min\bigg\{\frac{b_y^{(m)}+b_x^{(m)}}{2\bd(A)},\, \frac{\pers(B)}{\pers(A)}\bigg\}$ and $\displaystyle c_A = \max\bigg\{\frac{2\bd(B)}{a^{(m)}_x+a^{(m)}_y},\, \frac{\pers(B)}{\pers(A)}\bigg\}$.  Then
$$
\Theta(c) = \begin{cases}
\text{ is constant at } \pers(B) & \text{ for } 0\le c \le c_B;\\
\text{ attains the minimum } \overline{d_D}(A, B) & \text{ for } c_B\le c< c_A;\\
\text{ is linear with } c\cdot\pers(A) & \text{ for } c_A \le c.
\end{cases}
$$
\end{proposition}

This behavior is consistent with our findings in applications to real data as we will see in Section \ref{sec:application}.

\subsection{Computing Dilation-Invariant Bottleneck Distances}

Using the same framework as for the dilation-invariant bottleneck dissimilarity previously, we can similarly derive a direct search algorithm for the computation of dilation-invariant bottleneck distance $D_S$ as follows: following Theorems \ref{thm:left-end} and \ref{thm:right-end}, for any two persistence diagrams $A$ and $B$, set
	\begin{align*}
	    s_B &=  \log(b^{(m)}_y)-\log(\bd(A))-d_\infty(\log(A),\log(B))\\
	    s_A &=  \log(\bd(B))-\log(a^{(m)}_y)+d_\infty(\log(A),\log(B))
	\end{align*}
and search for the optimal shift $s^*$ in the interval $[s_B, s_A]$.
 
An alternative approach to computing the dilation-invariant bottleneck distance is to adopt the  kinetic data structure approach proposed by \cite{sheehy2018computing} to compute the shift-invariant bottleneck distance.  This method iteratively updates the matching value until the event queue is empty.  When the loop ends, it attains the minimum which is exactly $D_S$.  The algorithm runs in $O(n^{3.5})$ time where $n$ is the number of points in the persistence diagrams.

We remark here that our proposed direct search algorithm runs in $O(Nn^{1.5}\log n)$ time, which is a significant improvement over the kinetic data structure approach to compute the shift-invariant bottleneck distance.  On the other hand, the optimal value that we obtain is approximate, as opposed to being exact.

\subsubsection{Performance Comparison}

We compared both algorithms on simulated persistence diagrams; results are displayed in Table \ref{tab:runtime}. Note that in general, persistence diagrams contain points with 0 birth time which will cause an error of a $-\infty$ value when taking logs; we bypassed this difficulty by adding a small constant to all persistence intervals (in this test, the constant we used was $10^{-10}$).  This correction may be used for a general implementation of our method to compute dilation-invariant bottleneck distances. For the direct search algorithm, we fixed the number of partitions to be 100 in each test and used the inline function \texttt{bottleneck-distance} in \texttt{GUDHI}. The time comparison experiments were run on a PC with Intel(R) Core(TM) i7-6500U CPU @ 2.50GHz and 4 GB RAM; runtimes are given in wall-time seconds.

\begin{table}[htbp]
	\centering
	\caption{Comparison of the Direct Search (DS) Algorithm and the Kinetic Data Structure (KDS) Algorithm}
	\begin{tabular}{|c|c|c|c|c|c|}
		\hline
		Number of Points & Original Distance & KDS Distance & DS Distance & KDS Runtime & DS Runtime \\
		\hline
		32& 8.9164 & 6.9489 & 7.0635 & 1.8105 & 0.1874\\
		\hline
		64& 7.4414 & 6.9190 & 6.9572 & 17.0090& 0.3280\\
		\hline
		128& 0.5545 & 0.5407 & 0.5407 & 146.1578 & 0.6547\\
		\hline
		256& 1.6291 & 1.3621 & 1.3800 & 767.4551 & 1.4373 \\
		\hline
		512 & 2.3304& 1.9330 & 1.9459 & 12658.2584 &3.3468\\
		\hline
		1785 & 6.3851 & 3.1925 & 3.2511 & 614825.1576 & 51.1568\\
		\hline  
	\end{tabular}
\label{tab:runtime}
\end{table}

An immediate observation is the vastly reduced runtime of our proposed direct search algorithm over the kinetic data structure-based algorithm, particularly when the number of persistence points in the diagrams increases.  An additional observation is that our direct search algorithm tends to estimate conservatively and systematically return a distance that is larger than that returned by the kinetic data structure approach, which is known to give the exact value.  The difference margins between the two distances also tends to decrease as the number of persistence points in the diagrams increases.  Overall, we find that our direct search approach finds accurate results that are consistent with the exact values far more efficiently; when exact values are not required, the accuracy versus efficiency trade-off is beneficial.

We plotted the differences in the logs of the runtime and the number of persistence points to evaluate the time complexity of the two algorithms in practice, shown in Figure \ref{fig:empirical-complexity}.   The slope of each straight line gives the order of the computational complexity for each algorithm. We see that in our set of experiments, the empirical complexity is actually lower than, but still close to, the theoretical complexity bounds.   
\begin{figure}[htbp]
	\centering
	\includegraphics[width=0.65\textwidth]{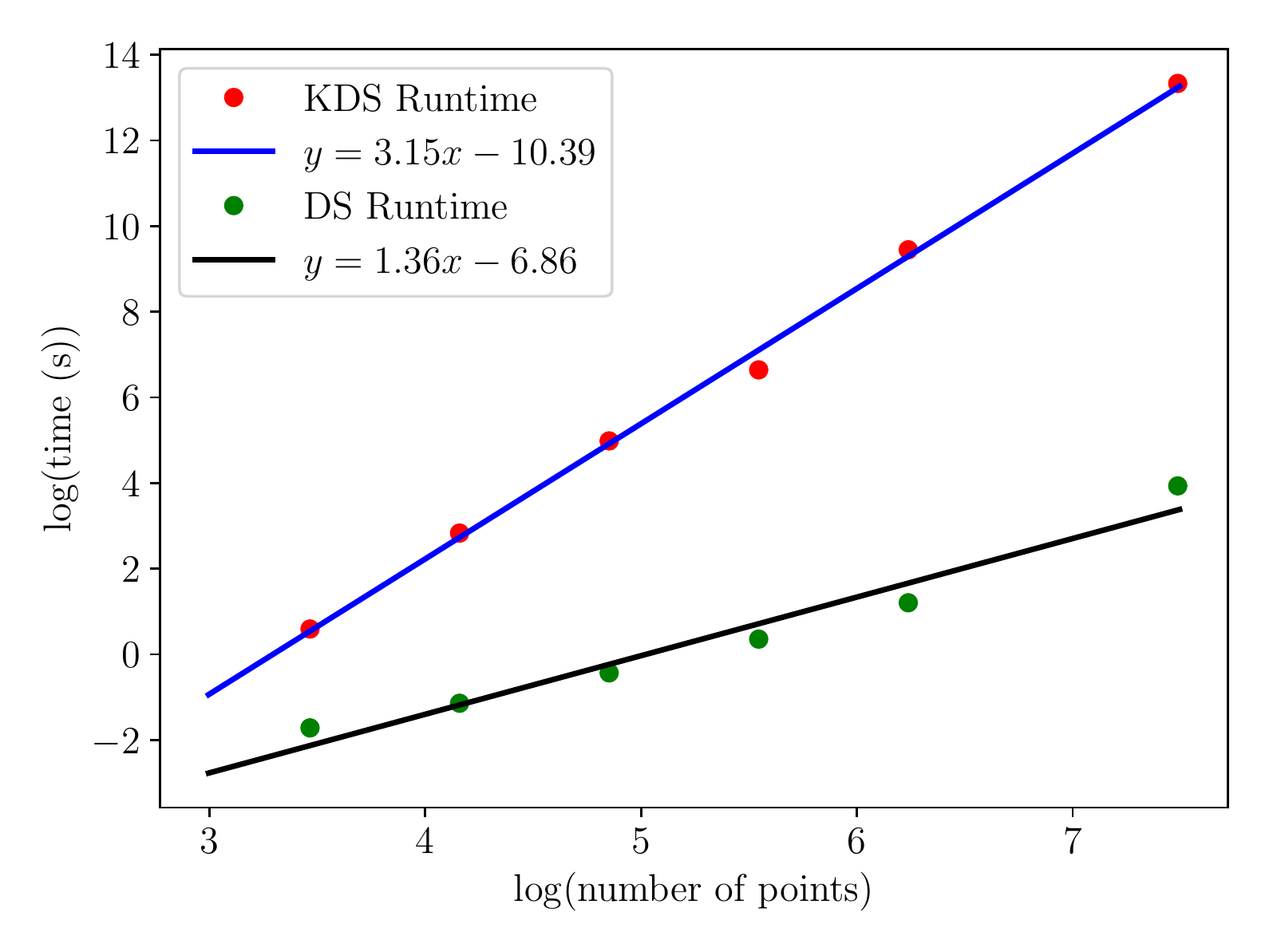}
	\caption{Plot of $\log({\rm time})-\log(\#{\rm points})$ for both search algorithms based on simulated persistence diagrams. The slopes match the theoretical orders of complexity.}
	\label{fig:empirical-complexity}
\end{figure} 

\subsection*{Software and Data Availability}

Software to compute the dilation-invariant bottleneck dissimilarity and to reproduce all results in this paper are available at \url{https://github.com/YueqiCao/Topological Retrieval}.


\section{Application to Data and Results}
\label{sec:application}

We present two applications to data: one is an exploratory analysis, where we evaluate our approach on two datasets that have a known hierarchical and network structure; the other is an example of inference where we perform topology-based information retrieval based on the dilation-invariant bottleneck distance.

\subsection{Exploratory Analysis: Database Representation for Information Retrieval}

We apply persistent homology to study the database structures in their manifold-embedded representations rather than on the graphs themselves.  Although persistent homology has been computed directly on graphs and networks (for instance, in applications in neuroscience \citep[e.g.,][]{10.1162/netn_a_00094}), manifolds may be equipped with a larger variety of metrics than graphs; moreover, on a manifold, distances may be defined between arbitration points along any path, while on a graph, we are constrained to traveling only along the graph edges.  Since the metric is a fundamental component in computing persistent homology, it is desirable to work in a setting where we may study a variety of metrics in a computationally efficient manner to obtain a better understanding of the persistence structure of the databases.

In IR, a dataset $Ds := \{d_1,...,d_N\}$ consisting of $N$ data points is embedded in a latent manifold $M$ via a model $f$. A query $q$ is then embedded in the same space and compared against all or a subset of the elements of $Ds$ via a distance metric $d$. The $m$ closest matches are then returned. Formally, the returned set $Rs$ is defined as $ Rs = \min_m \{d(f(Ds),f(q))\}$. The local structure of the space as well as the global connectivity of the dataset $Ds$ is of paramount importance in IR. 
Persistent homology is able to characterize both the global and local structure of the embedding space.
 
Following \citet{Nickel2017}, we embed the hierarchical graph structure of our datasets onto a Euclidean manifold and a Poincar\'{e} ball by minimizing the loss function 
\begin{equation}
    \mathcal{L} = \sum_{( u,  v) \in \mathcal{H}} \log \frac{e^{-d( u,  v)}}{\sum_{ v'\in \mathcal{N}(u)}e^{-d( u, v')}},
    \label{eq:loss}
\end{equation}
where $d$ is the distance of the manifold $M$, $\mathcal{H}$ is the set of hypernymity relations, and $\mathcal{N}=\{ v' \mid ( u, v') \notin D\}\cup\{ v\} $, thus the set of negative examples of $ u$ including $ v$.

In the case of the Euclidean manifold, we use two metrics: the standard $L_2$ distance and the cosine similarity,
$$
sim( x, y)=\frac{ x y}{\| x\|\| y\|},
$$
which has been used in inference in IR \cite{scarlini-etal-2020-sensembert, yap-etal-2020-adapting}.
On the Poincar\'{e} ball, the distance is given by 
\begin{equation}
\label{eq:poincare_dist1}
d_{P}( x, y)=\cosh^{-1}\left(1+\frac{2\| x- y\|^2}{(1-\| x\|^2)(1-\| y\|^2)}\right),
\end{equation}
where $\|\cdot\|$ is the standard Euclidean norm.  We optimize our model (\ref{eq:loss}) directly on each of the manifolds with respect to these three metrics.  

For the Poincar\'{e} ball, in addition to studying the direct optimization, we also study the representation of the Euclidean embedding on the Poincar\'{e} ball.  The Euclidean embedding is mapped to the Poincar\'{e} ball using the exponential map \citep{Mathieu2019}:
\begin{equation*}
\label{expmapeq}
\exp_{ x}^c( z) = 
 x\oplus_c \left[
\tanh\left(\frac{\sqrt{c}\,\lambda_{ x}^c\| z\|}{2}\right)\frac{\hat { z}}{\sqrt{c}}
\right],
\end{equation*}
where 
$\hat { z} :=  z/|| z||$, $\lambda_{ x}^c := 2/(1-c\| x\|^2)$, and $\oplus_c$ denotes M{\"o}bius addition under the curvature $c$ of the manifold \citep{Ganea2018networks}. Here, we take $c = -1$.

We calculate the VR persistent homology of the four embeddings to evaluate  the current IR inference approaches from a topological perspective.  Since we work with real-valued data, the maximum dimension for homology was set to 3.  All persistent homology computations in this work were implemented using \texttt{giotto-tda} \citep{tauzin2020giottotda}.  We note here that since VR persistent homology is computed, Ripser is available for very large datasets \citep{bauer2021ripser}.  Ripser is currently the best-performing software for VR persistent homology computation both in terms of memory usage and wall-time seconds (3 wall-time seconds and 2 CPU seconds to compute persistent homology up to 2 dimensions of simplicial complexes of the order of $4.5 \times 10^8$ run on a cluster, 3 wall-time seconds for the same dataset run on a shared memory system) \citep{otter2017roadmap}.

\paragraph{ActivityNet.}
ActivityNet is an extensive labeled video dataset of 272 human actions that are accompanied by a hierarchical structure, visualized in Figure \ref{fig:activity_net_network} \citep{caba2015activitynet}. This dataset has been previously used in IR \citep{Long_2020_CVPR}. 

\begin{figure}
\centering
  \includegraphics[scale=0.65]{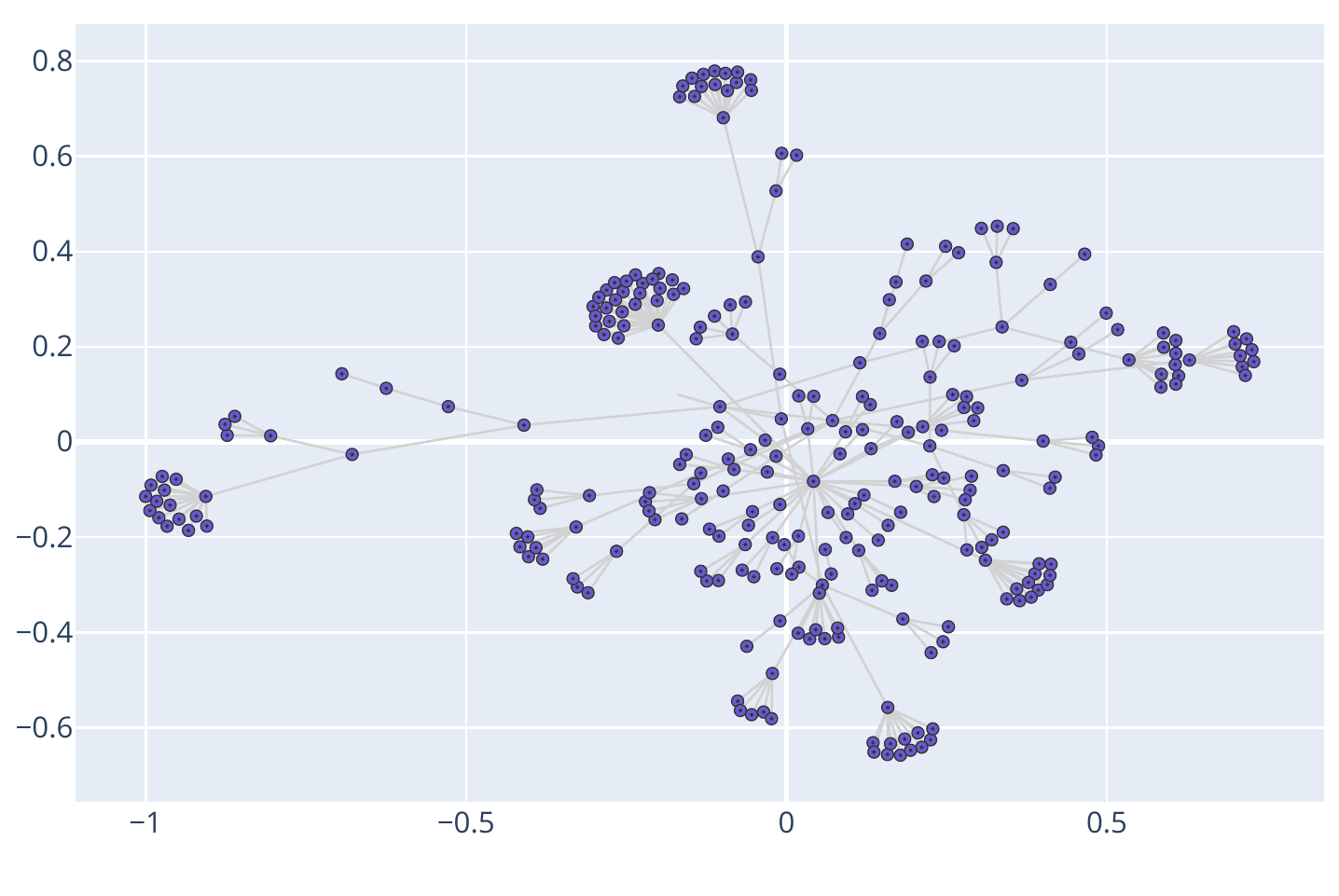}
  \caption{Visualization of ActivityNet Dataset.  Note that hierarchical structure of  the dataset naturally exhibits interesting topological properties such as cycles, clusters, and a tree-like structure.}
\label{fig:activity_net_network}
\end{figure}

\paragraph{Mammals of WordNet.}
WordNet is a lexical database that represents the hierarchy of the English language in a tree-like form, which has been extensively used in NLP \citep{10.1145/219717.219748}. In our application, we focus on the subportion of WordNet relating to mammals. We pre-process this dataset by performing a transitive closure \citep{Nickel2017}.\\


\noindent\textbf{Embedding the Data.}
Both datasets were embedded in a 5-dimensional space, following \citet{Nickel2017}.  
For both datasets, we trained the models for 1200 epochs with the default hyperparameters. For ActivityNet, we were able to achieve hypernimity mAP of $94.3\%$ and $97.2\%$ for the Euclidean manifold and Poincar\'e ball, respectively; while for the Mammals data, we achieved $45\%$ and $82.3\%$, respectively. These results are consistent with \citet{Nickel2017}.

\subsubsection{Topology of the Euclidean Embeddings}

The persistence diagrams of the Euclidean embeddings under both metrics and for both datasets are given in Figure \ref{fig:pers_eucl}.  We see that up to scaling, the persistence diagrams for both embeddings into Euclidean space with respect to the $L_2$ metric and cosine similarity are similar in terms of location and distribution of $H_0$, $H_1$, and $H_2$ points on the persistence diagrams.  We quantify this similarity by computing dilation-invariant bottleneck dissimilarities. The graphs of the search procedure for each dataset are given in Figure \ref{fig:bottleneck}.  We first note that the standard bottleneck distances between the $L_2$ and cosine similarity persistence diagrams are $1.642$ for the ActivityNet dataset, and $0.6625$ for the Mammals dataset.  The dilation-invariant bottleneck dissimilarities are achieved at the minimum value of the bottleneck distance in our parameter search region.  The dilation-invariant bottleneck dissimilarities from cosine similarity to $L_2$ are $1.311$ for the ActivityNet dataset and $0.3799$ for the Mammals dataset.

The dilation parameter for Euclidean embeddings can be alternatively searched according to the scale of $L_2$ embedded points as follows.  Since mutual distances under the cosine embedding are always bounded by 1, we can regard the induced finite metric space as a baseline. For the $L_2$ embedding, mutual dissimilarities are affected by the scale of the embedded point cloud, i.e., if we multiply the points by a factor $R$, the induced metric will also dilate by a factor $R$.  In experiments, the scale is $18.13$ for the $L_2$ embedding of ActivityNet and $3.188$ for Mammals, which are consistent with dilation parameters we found.


The visual compatibility of the persistence diagrams together with the small dilation-invariant bottleneck distances indicate that the topology of the database is preserved in the Euclidean embedding, irrespective of the metric.  

\begin{figure*}[]
  \centering
  \includegraphics[scale=0.375]{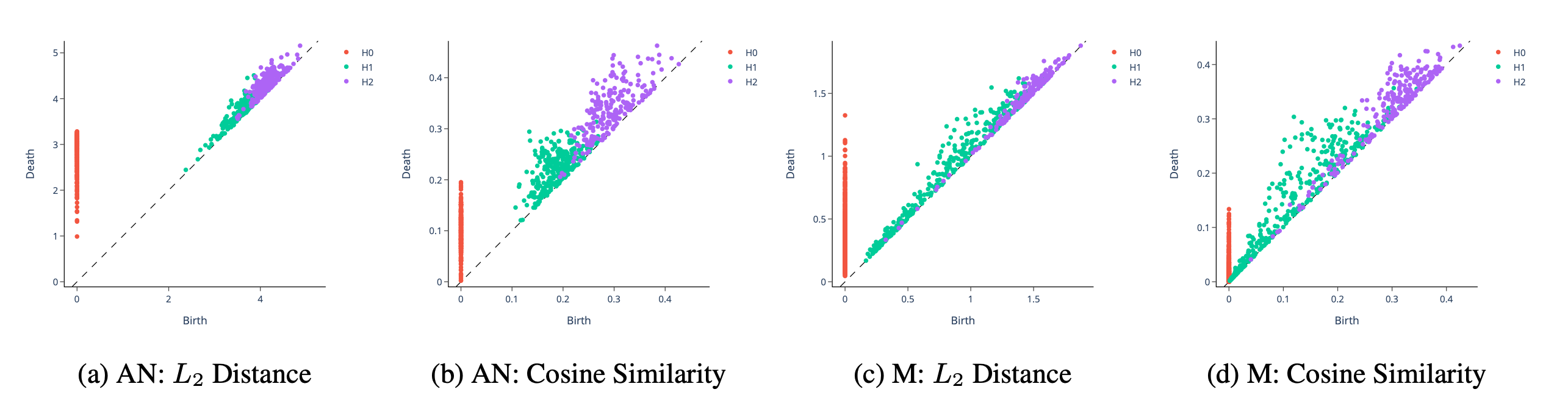}
  \caption{Persistence Diagrams of Euclidean Embeddings.  AN: ActivityNet, M: Mammals. \label{fig:pers_eucl}}
\end{figure*}

\subsubsection{Topology of the Poincar\'{e} Ball Embeddings}

The persistence diagrams for the Poincar\'{e} ball embeddings for both datasets are given in Figure \ref{fig:pers_poinc}.  For the ActivityNet dataset, we see that both embeddings fail to identify the $H_1$ and $H_2$ homology inherent in the dataset (see Figure \ref{fig:activity_net_network}), which, by contrast, the Euclidean embeddings successfully capture.  For the Mammals dataset, only the exponential map embedding into the Poincar\'{e} ball is able to identify $H_1$ and $H_2$ homology, but note that these are all points close to the diagonal.  Their position on the diagonal indicates that the  persistence of their corresponding topological features is not significant: these features are born and die almost immediately, are therefore likely to be topological noise.  We now investigate this phenomenon further.

\begin{figure*}[]
  \centering
\includegraphics[scale=0.375]{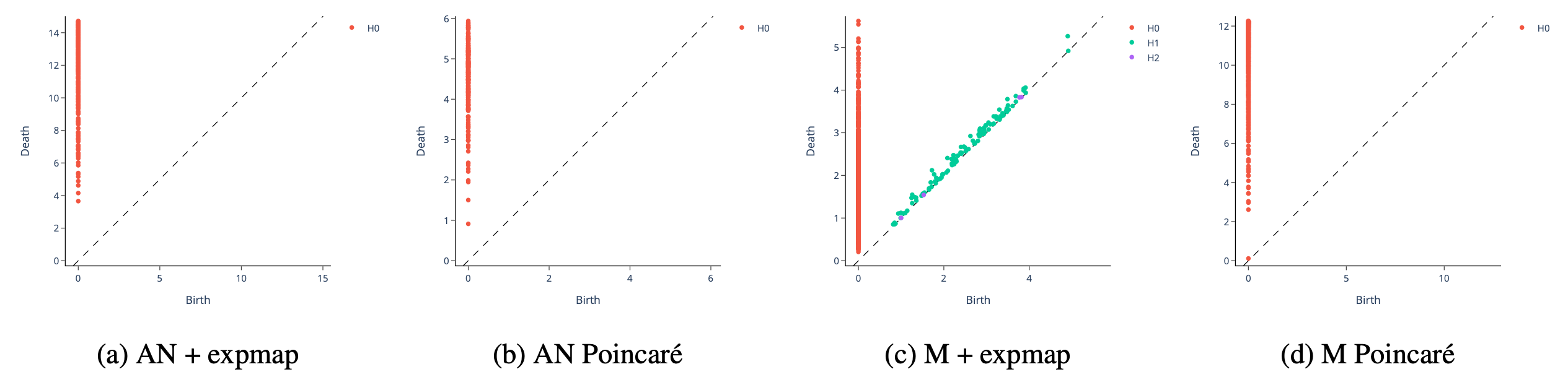}
  \caption{Persistence Diagrams of Poincar\'e Embeddings.  AN: ActivityNet, M: Mammals, expmap signifies Euclidean optimization followed by an exponential map embedding into the Poincar\'{e} ball, while  Poincar\'e signifies direct optimization on the Poincar\'{e} ball. 
  \label{fig:pers_poinc}}
\end{figure*}

From (\ref{eq:poincare_dist1}), we note that an alternative representation for the distance function maybe obtained as follows.  Set
\begin{equation}
\label{eq:poincare_H}
A(x,y) := 1+\frac{2\|x-y\|^2}{(1-\|x\|^2)(1-\|y\|^2)},
\end{equation}
then (\ref{eq:poincare_dist1}) may be rewritten as
\begin{equation}
\label{eq:poincare_dist2}
d_P(x,y)=\log (A(x,y)+\sqrt{A(x,y)^2-1}).
\end{equation}

\begin{figure*}[]
\centering
\subfigure[AN Poincar\'{e}]{\includegraphics[width=0.24\linewidth]{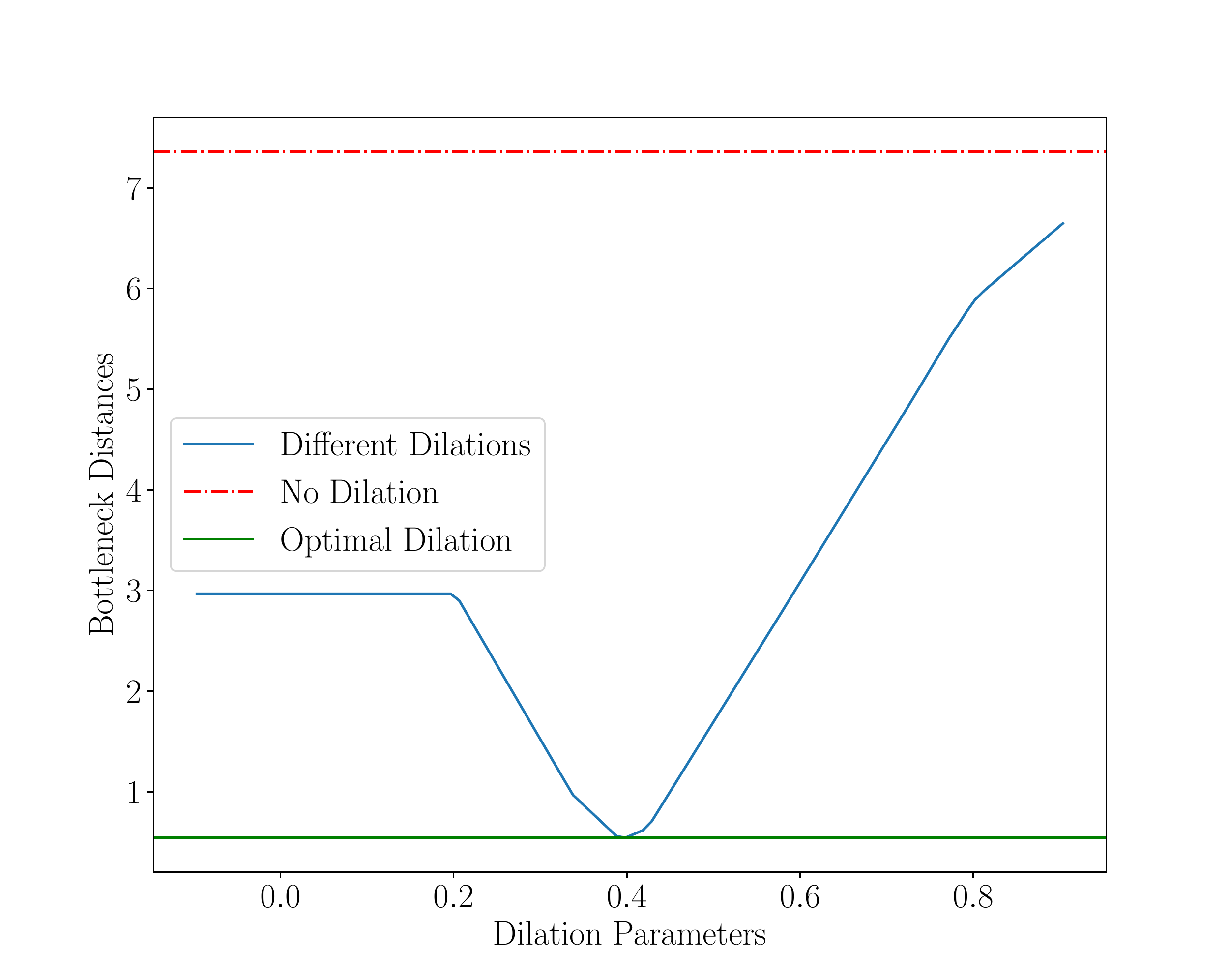}}
\subfigure[M Poincar\'{e}]{\includegraphics[width=0.24\linewidth]{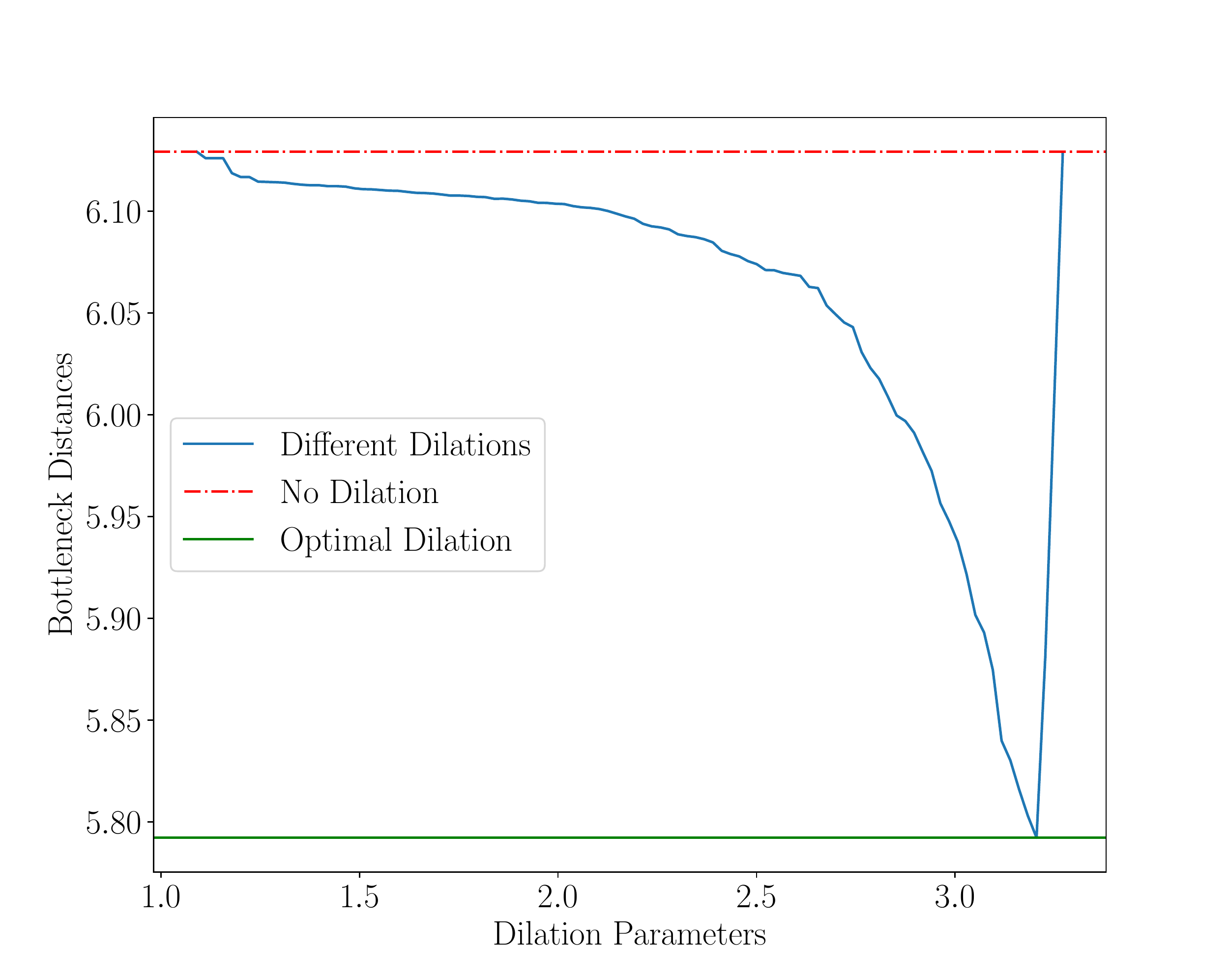}}
\subfigure[AN Euclidean]{\includegraphics[width=0.24\linewidth]{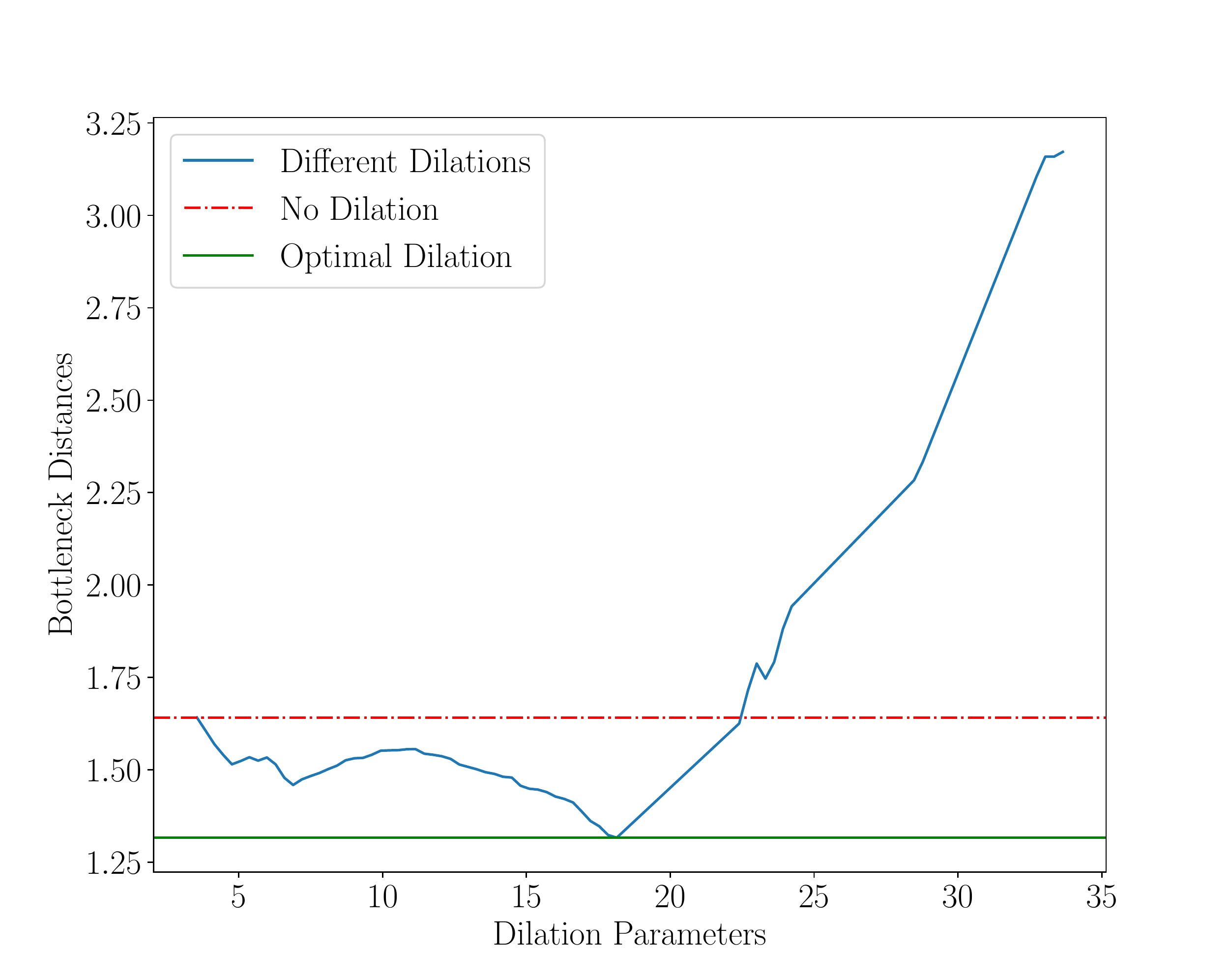}}
\subfigure[M Euclidean]{\includegraphics[width=0.24\linewidth]{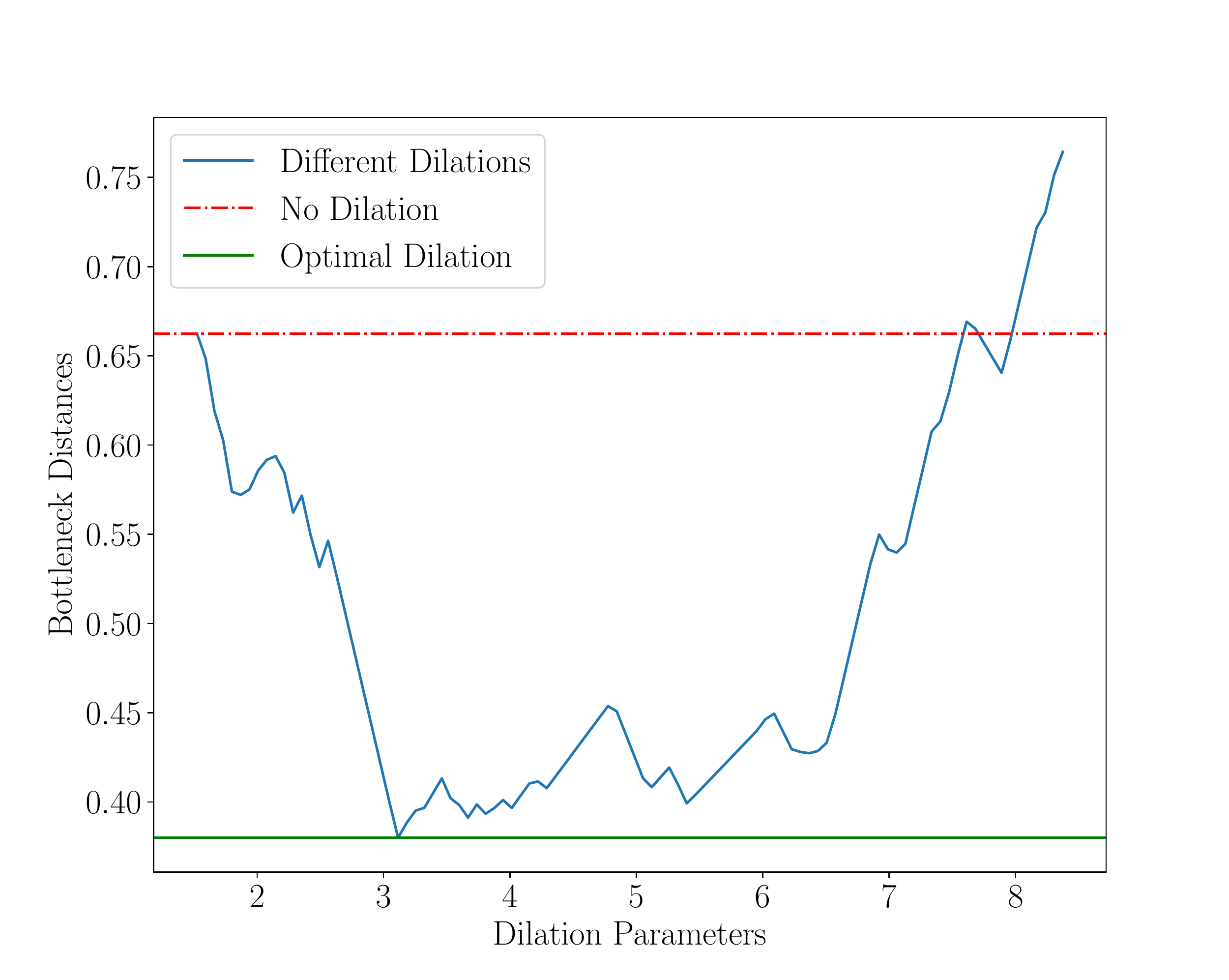}}
\caption{Dilation-Invariant Bottleneck Distances Between Embeddings.  AN: ActivityNet, M: Mammals}
\label{fig:bottleneck}
\end{figure*}

We assume that $\| x-y \|$ has a strictly positive lower bound and study the case when both points $x,y$ lie near the boundary, i.e., both $\| x\| \approx 1$ and $\|y \| \approx 1$.  Notice, then, that (\ref{eq:poincare_H}) and (\ref{eq:poincare_dist2}) are approximated by
\begin{align}
A(x,y) & \approx \frac{1}{2}\frac{\|x-y\|^2}{(1-\|x\|)(1-\|y\|)}, \nonumber\\
d_P(x,y) & \approx \log(2A(x,y)) \nonumber\\
& \approx2\log(\|x-y\|)-\log(1-\|x\|)-\log(1-\|y\|). \label{eq:poincare_dist3}
\end{align}
Notice also that
\begin{equation}
\label{eq:poincare_dist0}
d_P(x,0) \approx \log2 -\log(1-\|x\|). 
\end{equation}
Combining (\ref{eq:poincare_dist3}) and (\ref{eq:poincare_dist0}) yields the following approximation for the Poincar\'{e} distance:
\begin{equation}
\label{eq:poincare_dist4}
d_P(x,y)\approx d_P(x,0)+d_P(y,0)+2\log(\|x-y\|/2).
\end{equation}
For points near the boundary, i.e., $d_P(x,0)=R\gg 1$, (\ref{eq:poincare_dist4}) gives a mutual distance of approximately $2R$ between any given pair of points. 

Recall from Definition \ref{def:VRcomplex} that a $k$-simplex is spanned if and only if the mutual distances between $k+1$ points are less than a given threshold $r$.  The implication for discrete points near the boundary of the Poincar\'{e} ball is that either the threshold is less than $2R$, and then the VR complex is a set of 0-simplices; or the threshold is greater than $2R$, and then the VR complex is homeomorphic to a simplex in higher dimensions.  That is, the VR filtration jumps from discrete points to a highly connected complex; this behavior is illustrated in Figure \ref{fig:poincare_ex}.

\begin{figure*}[ht]
\vskip 0.2in
\begin{center}
\centerline{\includegraphics[width=\columnwidth]{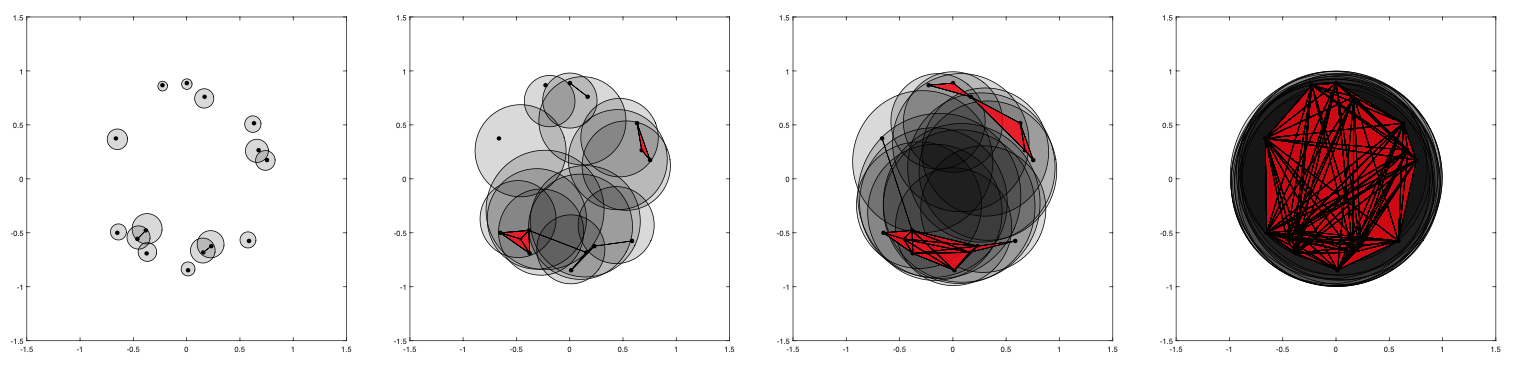}}
\caption{Illustration of connectivity behavior in a VR complex on the Poincar\'{e} ball equipped with the Poincar\'{e} distance.  For points sampled near the boundary, the connectivity of the VR complex grows slowly at low threshold values (small radii) and quickly becomes highly connected as the threshold grows and we approach the approximation (\ref{eq:poincare_dist4}) for distances between points.}
\label{fig:poincare_ex}
\end{center}
\vskip -0.2in
\end{figure*}

\begin{figure}[ht]
\begin{center}
\subfigure[Empirical CDF]{\includegraphics[width = 0.45\linewidth]{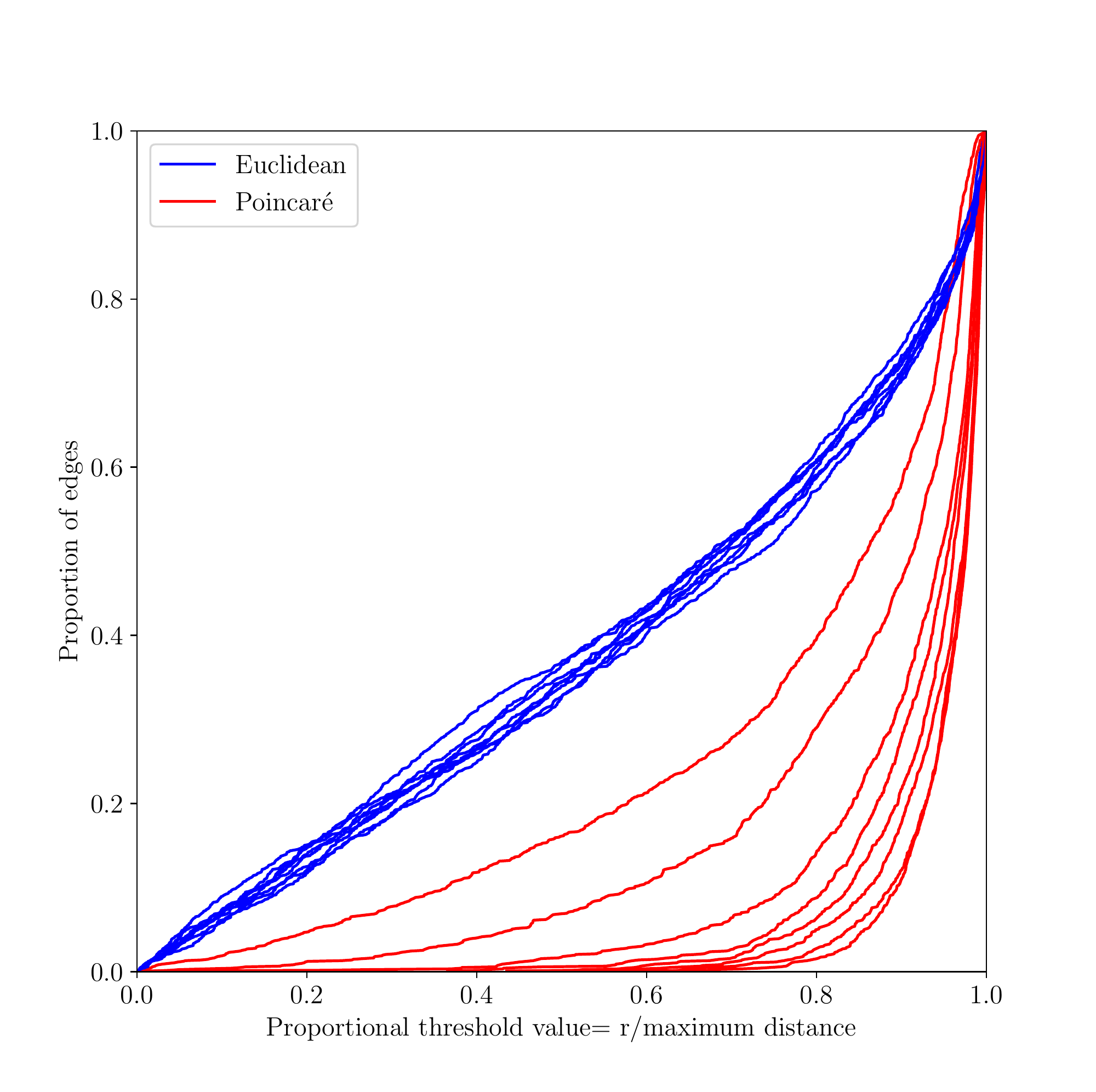}}
\subfigure[Population CDF]{\includegraphics[width=0.45\linewidth]{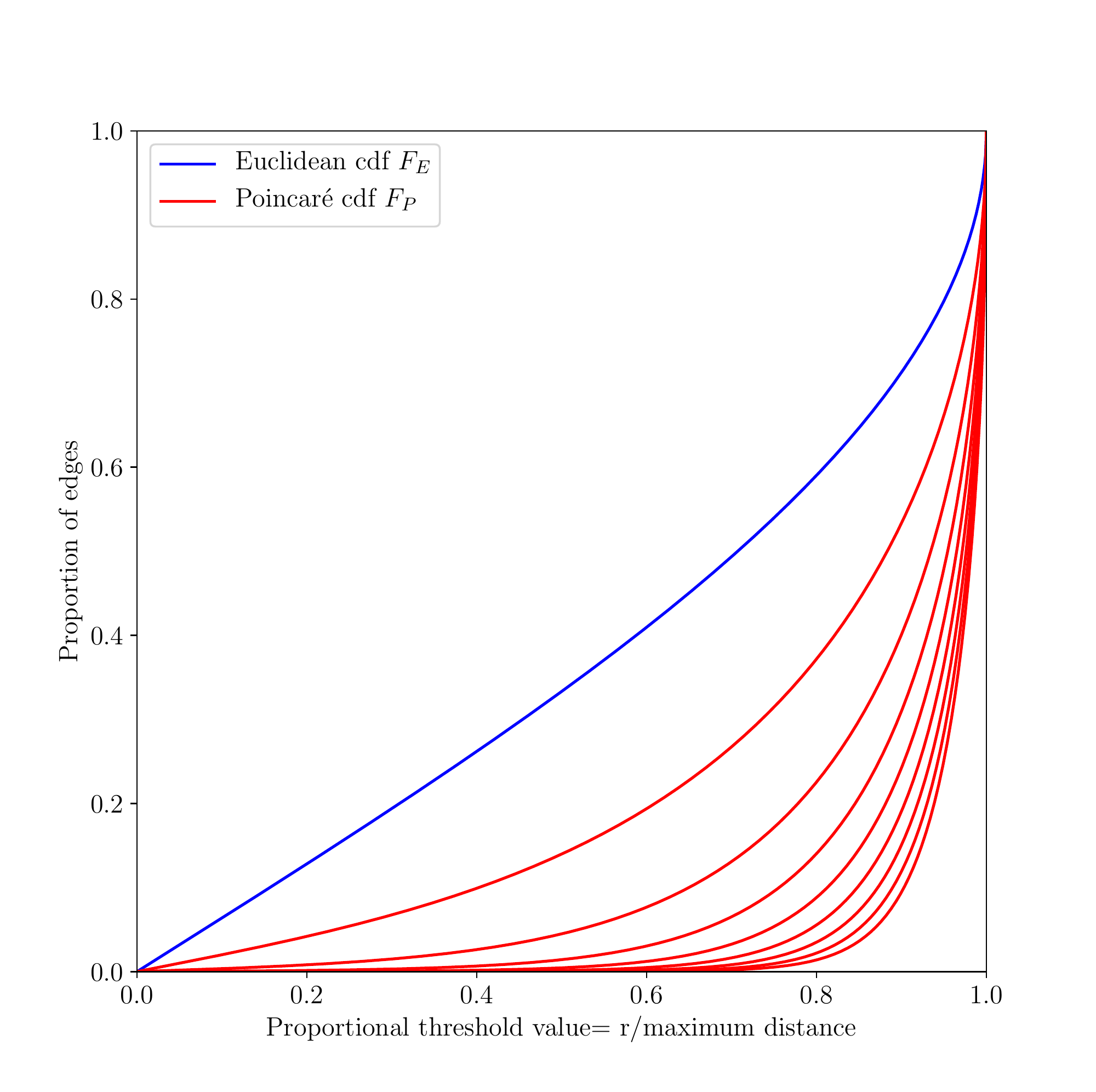}}
\caption{Cumulative distribution functions (cdf) of edges in a VR complex with respect to Poincar\'{e} and Euclidean distances. (a) Proportion of edges in a VR filtration with respect to the Poincar\'{e} distance on the Poincar\'{e} ball (blue lines) and the Euclidean distance on Euclidean space (red lines), computed (\ref{eq:percentage_edges}), against scaled VR threshold values Definition \ref{def:VRcomplex}. (b) Population cdf as derived in Theorems \ref{thm:cdf_euclidean} and \ref{thm:cdf_poincare}.  }
\label{fig:poincare_dist}
\end{center}
\vskip -0.2in
\end{figure}

This phenomenon of sharp phase transition is observable in our applications, which we further explore through numerical experiments and display in Figure \ref{fig:poincare_dist}.  Here, each curve corresponds to a numerical experiment.  We also plot the corresponding connectivity of VR complexes with respect to the Euclidean distance for comparison.  For a given experiment, we draw the distance cumulative distribution function as follows: first, we normalize distances and sort them in ascending order; next, for the $k$th element $R_k$, we compute the proportion of connected pairs as
\begin{equation}
\label{eq:percentage_edges}
\frac{\#\{(x_i,x_j) \mid d_P(x_i,x_j)\le R_k\}}{\#\{(x_i,x_j)\}}.
\end{equation}
From the cumulative distribution plot in Figure \ref{fig:poincare_dist}, we see this rapidly increasing connectivity behavior at higher threshold values in repeated numerical experiments.  In each experiment, points are sampled closer and closer to the boundary of the ball, and we see a progressively sharper jump in connectivity of the VR complex the closer to the boundary the points lie for the Poincar\'{e} distance.  For the extreme case with points sampled closest to the boundary of the ball, when the VR threshold value is at $90\%$ of the maximum distance, there are less than $10\%$ pairs of points that are connected.  This percentage sharply and dramatically increases at thresholds from 90\% to the maximum distance, from $10\%$ to $100\%$ of pairs of points being connected.  The phase change explains why we cannot see long persistence (in terms of points appearing far away from the diagonal) in dimension 1 or 2 in Figure \ref{fig:pers_poinc}.  The connectivity of the VR complex under the Euclidean distance, in contrast, grows linearly with the threshold.  Moreover, we see that the connectivity is stable and robust to the sampling location of the points since there is very little variability in the curves; the linear growth in connectivity is unaffected by the location of the sampled points.  We now formalize these observations.

\paragraph{Statistical Properties of VR Edges on the Poincar\'{e} Ball.}

Since edges in a VR complex are determined by distances between vertices, we are interested in studying the distributions of the distances measured by the Euclidean (standard $L_2$ distance) $d_E$ and Poincar\'{e} $d_P$ metrics on the ball.  Let $\mathbb{S}^1(r)$ be a circle of radius $r$ (measured in Euclidean distance); then any point on $\mathbb{S}^1(r)$ takes the form of $re^{i2\pi\theta}$.  Let $\theta_X$ and $\theta_Y$ be two independent random variables uniformly distributed on $\mathbb{R}/\mathbb{Z} = [0,1)$; and set $X := re^{i2\pi\vartheta_X}$ and $Y := re^{i2\pi\theta_Y}$ so that $X$ and $Y$ are $\mathbb{S}^1(r)$-valued random variables.  We now characterize the distributions of $d_E(X,Y)$ and $d_P(X,Y)$.

\begin{lemma}
Let $\theta_X, \theta_Y \sim \mathrm{Unif}(\mathbb{R}/\mathbb{Z})$.  Then $Z := \theta_Y - \theta_X \sim \mathrm{Unif}(\mathbb{R}/\mathbb{Z})$.
\end{lemma}

\begin{proof}
For any $a \in [0,1]$, observe that
\begin{align*}
\mathbb{P}(Z \leq a) & = \sum_{n \in \mathbb{Z}} \mathbb{P}(n \leq \theta_Y - \theta_X \leq a + n)\\
& = \mathbb{P}(0 \leq \theta_Y - \theta_X \leq a) + \mathbb{P}(-1 \leq \theta_Y - \theta_X \leq a - 1)\\
& = \frac{1}{2} - \frac{1}{2}(1-a)^2 + \frac{1}{2}a^2 = a,
\end{align*}
as desired.
\end{proof}

The cumulative distribution function (cdf) of the Euclidean distance between $X$ and $Y$ can be explicitly derived as follows.

\begin{theorem}
\label{thm:cdf_euclidean}
Let $F_E :[0,1] \rightarrow [0,1]$ denote the cdf of the normalized Euclidean distance $\frac{1}{2r}d_e(X,Y)$ between $X$ and $Y$.  Then for all $t \in [0,1]$,
$$
F_e(t) = \frac{2}{\pi}\sin^{-1}(t).
$$
\end{theorem}

\begin{proof}
First notice that
\begin{align}
d_E(X,Y) & = r \sqrt{(\cos(2\pi\theta_X) - \cos(2\pi\theta_Y))^2 + (\sin(2\pi\theta_X) - \sin(2\pi\theta_Y))^2} \nonumber\\
& = r \sqrt{2 - 2 \cos(2\pi(\theta_Y - \theta_Z))} \nonumber\\
 & = 2r\sin(\pi Z). \label{eq:cdf_eq_l2}
\end{align}
Then we have
\begin{equation}
\label{eq:sinpiz}
F_e(t) = \mathbb{P}\bigg( \frac{1}{2r}d_E \leq t \bigg) = \mathbb{P}(\sin(\pi Z) \leq t) = \frac{2}{\pi}\sin^{-1}(t).
\end{equation}
\end{proof}

Notice that the cdf does not depend on $r$, which is consistent with what we observe in numerical experiments given in Figure \ref{fig:poincare_dist}.  We now derive the cdf for Poincar\'{e} distance $d_P$ between $X$ and $Y$.

\begin{theorem}\label{thm:cdf_poincare}
Let $M := \max \{d_P(X,Y)\}$ and let $F_P : [0,1] \rightarrow [0,1]$ be the cdf of the normalized Poincar\'{e} distance $\frac{1}{M}d_P(X,Y)$.  Then for all $t \in [0,1]$,
$$
F_P(t) = \frac{2}{\pi}\sin^{-1}\bigg( \frac{1}{2} \sinh \bigg( \frac{tM}{2} \bigg) \bigg( \frac{1}{r} - r \bigg) \bigg).
$$
\end{theorem}

\begin{proof}
We have
$$
d_P(X,Y) = \cosh^{-1}\bigg( 1 + \frac{2d_E(X,Y)^2}{(1-r^2)^2} \bigg) = \cosh^{-1}\bigg( 1 + \frac{8r^2\sin^2(\pi Z)}{(1-r^2)^2} \bigg) = 2\sinh^{-1}\bigg( \frac{2\sin(\pi Z)}{\frac{1}{r} - r} \bigg),
$$
by substituting (\ref{eq:cdf_eq_l2}) above.  We then have that $\displaystyle M = 2\sinh^{-1}\bigg( \frac{2}{\frac{1}{r} -r} \bigg)$ and compute, using (\ref{eq:sinpiz}):
\begin{align*}
F_P(t) & = \mathbb{P}(d_P(X,Y) \leq tM)\\
& = \mathbb{P}\bigg(d_E(X,Y)^2\le\frac{(1-r^2)^2(\cosh(tM)-1)}{2}\bigg)\\
& = \mathbb{P}\bigg(\sin(\pi Z) \leq \frac{1}{2}\sinh\bigg( \frac{tM}{2} \bigg)\bigg( \frac{1}{r} - r \bigg) \bigg)\\
& = \frac{2}{\pi}\sin^{-1}\bigg( \frac{1}{2} \sinh \bigg( \frac{tM}{2} \bigg) \bigg( \frac{1}{r} - r \bigg) \bigg).
\end{align*}
\end{proof}

Notice that contrary to the Euclidean case given by Theorem \ref{thm:cdf_euclidean} above, here, the cdf depends on the radius $r$.  More precisely, let $r_n$ be a sequence of radius values converging to 1 and here, let $R_n := \frac{1}{M_n} d_P(X_n, Y_n)$, where $M_n, X_n, Y_n$ are the corresponding $M, X, Y$ values associated with the radius $r_n$; let $F_P^{(n)}$ be the cdf of each $R_n$.  Then
$$
F_P^{(n)}(t) \rightarrow F^{\infty}(t) = \begin{cases}
1, & t = 1;\\
0, & 0 \leq t < 1.
\end{cases}
$$
In other words, we have that $R_n$ converges in distribution to the Dirac delta $\delta_1$.\\



Note that although the topology of the databases is not captured in the Poincar\'{e} ball embeddings, the dilation-invariant bottleneck dissimilarity for the ActivityNet embedding is nevertheless small, with a value of $0.5481$ and corresponding dilation parameter value of $3.986$.  We do not see this coincidence for the Mammals dataset, however: the dilation-invariant bottleneck dissimilarities between the persistence diagram arising from the exponential map embedding and the direct Poincar\'{e} embedding remains large at $5.808$ with a corresponding dilation parameter value of $3.182$.  This large dilation-invariant bottleneck dissimilarity is likely due to the difference in occurrences of homology between the two diagrams: the exponential map embedding picks up on spurious $H_1$ and $H_2$ homology, while the direct Poincar\'{e} embedding does not.  There is no immediate interpretation for the dilation parameter value and its effect on the dilation-invariant bottleneck dissimilarities, due to the nonlinear and asymmetric behavior of the comparison measure and derived approximations (see Figure \ref{fig:poincare_ex}).

Overall, we find that the higher-order connectivity and hierarchy of the databases was preserved in the Euclidean embeddings, but not in the Poincar\'{e} ball embeddings: $H_1$ and $H_2$ homology corresponding to cycles and voids appear in the persistence diagrams of the Euclidean embeddings, and do not in the Poincar\'{e} ball embeddings.

\subsection{Symmetry versus Asymmetry in Practice}
\label{sec:sym}	

Here, we give a demonstration with real data on the difference between assuming a symmetric versus asymmetric dilation-invariant bottleneck dissimilarity as discussed previously in Section \ref{sec:weak_isom}. We compute both the asymmetric \eqref{eq:DI-dissimilarity} and symmetric \eqref{eq:DI-symmetric} dilation-invariant bottleneck dissimilarities for the persistence diagrams of the ActivityNet and Mammals datasets to the Euclidean embedding with $L_2$ metric of the Mammals dataset (which we denote as M $L_2$); the results are given in Table \ref{tab:sym-vs-asym}.
  
  \setlength{\extrarowheight}{3pt}
  \begin{table}[htbp]
      \centering
      \begin{tabular}{|c|c|c|c|c|c|c|c|}
      \hline
          & AN Cosine & AN $L_2$ & AN+expmap & AN Poincar\'{e} & M Cosine & M+expmap & M Poincar\'{e}\\
          \hline
          $\overline{d_D}$ & 0.38 & 0.45 & 0.57 & 0.49 & 0.39 & \textbf{0.21} & \textbf{0.44} \\
          \hline
          $\overline{d_{\mathrm{sym}}}$ & \textbf{0.23} & 0.93 & 3.02 & 1.24 & \textbf{0.23} & 0.61 & \textbf{3.27}\\
          \hline
      \end{tabular}
      \caption{Symmetric and asymmetric dilation-invariant bottleneck dissimilarities between persistence diagrams of all embeddings to the persistence diagram of the Euclidean with $L_2$ distance embedding of the Mammals dataset.}
      \label{tab:sym-vs-asym}
  \end{table}
  
Using the asymmetric dilation-invariant bottleneck dissimilarity, the most similar persistence diagram to that of the M $L_2$ embedding is the persistence diagram of the M+expmap embedding with a value of 0.21. It also finds that persistence diagrams of other embeddings of the Mammals dataset are close as well (0.39 and 0.44).
  
Using the symmetric dilation-invariant bottleneck dissimilarity, the most similar persistence diagrams to that of the M $L_2$ embedding are the persistence diagrams of the AN Cosine and M Cosine embedding (0.23). However, the persistence diagram of the M Poincar\'{e} dataset has a comparatively large dissimilarity (3.27). This is because $\overline{d_{D}}(A, B)$ depends on the scale of the second variable (persistence diagram) $B$. When comparing dissimilarities, it only makes sense if everything is compared under the same scale (i.e., fix $B$). If we symmetrize, the scales of $A$ and $B$ mix up and the result becomes uninterpretable. In the above example, the persistence diagram of the M Poincar\'{e} embedding has much larger scale than the persistence diagram of the M $L_2$ embedding. Though these two datasets have similar topology as shown using the asymmetric dilation-invariant bottleneck dissimilarity, the symmetric version is not able to identify this resemblance.

\subsection{Topological Information Retrieval: A Classification Problem}
\label{sec:classification}


We now turn to a problem of inference: specifically, we perform an IR task at the database level taking into account the inherent topology of the database.  For this task, we use seven medical imaging databases from MedMNIST \citep{medmnistv1}; summaries of these datasets are given in Table \ref{tab:medmnist}.  The task is the following: Given a set of points $N$, we seek the closest set of points $M$ where $N \ll M$ such that the two sets comprise representatives from the same category.  Essentially, this problem may be approached as a problem of topological classification.  Persistent homology has been successfully used in various classification problems and settings, including trajectory classification in robotics \citep{pokorny2014multiscale}, determination of parameters of dynamical systems \citep{adams2017persistence}, and sleep-wake classification in physiology \citep{10.3389/fphys.2021.637684}.  In particular, it has been used in medical imaging classification for endoscopy images \citep{DUNAEVA201613}, prostate cancer histopathology images \citep{prostate}, and liver biopsy tissue images \citep{TERAMOTO2020105614}.

\begin{table}[]
\centering
\begin{tabular}{|c|c|c|c|c|c|c|c|}
\hline
Dataset & Derma & Pneumonia & Retina & Breast & Organ\_Axial & Organ\_Coronal & Organ\_Sagittal \\ \hline
Classes & 7      & 2         & 5      & 2      & 11           & 11             & 11              \\ \hline
Samples & 10,015 & 5,856     & 1600   & 780    & 58,850       & 23,660         & 25,221          \\ \hline
\end{tabular}
\caption{Statistics of MedMnist datasets\label{tab:medmnist}}
\end{table}


Using an autoencoder based upon \cite{ResNetAE}, we embedded each dataset separately in Euclidean space.  We note in particular that the networks do not provide explicit information regarding the class of each input image.  Since the number of embedding points exceeded the reasonable computational time frame threshold for VR persistent homology (see Table \ref{tab:medmnist}), we obtained persistence diagram representatives of embeddings of each class as follows: we split the embeddings on non-overlapping sets of $N=\{500,600\}$ and computed the class-specific persistence diagram by the multiple subsampling scheme described by \cite{chazal2015subsampling}. Specifically, for a large dataset $X$ with $M$ points, we sample $m\ll M$ points randomly $B$-many times. The procedure yields $B$ point sets, $X_1,\cdots, X_B$, which then give the persistence diagrams $\mathrm{Dgm}(X_1),\cdots,\mathrm{Dgm}(X_B)$. The {\em Fr\'{e}chet mean} (also referred to as the Lagrangian barycenter \citep{lacombe2018large}) is given by 
$$
    \overline{\mathrm{Dgm}} = \mathop{\arg\min}_{Z\in\mathcal{D}}\frac{1}{B}\sum_{i=1}^Bd_2(Z,\, \mathrm{Dgm}(X_i))^2
$$
where $d_2$ is the 2-Wasserstein distance between persistence diagrams, which we compute following \cite{turner2014frechet}. In similar spirit to \cite{chazal2015subsampling}, we take the Fr\'{e}chet mean $\overline{\mathrm{Dgm}}$ to approximate $\mathrm{Dgm}(X)$. Recent results by \citet{cao2022approximating} give convergence results and show that the Fr\'{e}chet mean of persistence diagrams computed from subsampled datasets gives a valid approximation of the true persistence diagram of the larger dataset.

During inference, we randomly sample a subset of proportion $\{0.2,0.4,0.6,0.8\}$ of the class, compute VR persistence, and compare to the class-specific persistence diagrams using our  dilation-invariant bottleneck dissimilarity. The class of the most similar persistence diagram (i.e., the one with the smallest dilation-invariant bottleneck dissimilarity) is assigned to the query subset.

In Table \ref{tab:ir_results}, we show the classification results where we quote the top two accuracy of our predictions, i.e., whether or not the correct class is in the first or second predicted position based on the ordered similarity.  We note that with our the dilation-invariant bottleneck dissimilarity, we are able to predict the correct class with an accuracy of up to $87\%$. As expected, the more points we have in our query, the better the accuracy since the query persistence diagram is then a better representative of the true persistence diagram.  We further notice a slight drop in performance in some datasets for the $60\%$ proportion, which we hypothesize is due to suboptimal  sampling from the embeddings.  Notice here that since we are querying by persistence diagrams, we are effectively querying by sets, so there is no appropriate comparative measure in terms of other methods; the most appropriate comparison for our performance is the same IR task but using the standard bottleneck distance.  We find that our method matches or outperforms the standard bottleneck distance in the same IR task, especially in datasets with more classes when using less data to compute the persistence diagram.




\section{Discussion}
\label{sec:discussion}

In this work, we used persistent homology to quantify and study the topological structure of databases corresponding to their hierarchy and connectivity as an exploratory analysis.  We found that the Euclidean embeddings preserve the topology of the databases with respect to both the $L_2$ and cosine similarity metrics, while the Poincar\'{e} ball embeddings did not.  Moreover, we found that among topology-preserving embeddings, the topological structure captured by the persistence diagrams was consistent.  To quantify this effect, we introduced the dilation-invariant bottleneck dissimilarity and distance, which retains the resemblance between persistence diagrams and erases the distortion effect of the different metrics over the embeddings.  We found that dilation-invariant bottleneck distances between topology-preserving database representations are small in our exploratory analysis.  We further performed an IR task of classification using the dilation-invariant bottleneck distance and found a high level of accuracy of classification.  Our work is, to the best of our knowledge, the first application of TDA in IR.

Note that in our exploratory analysis and in studying the topology of the database representations, the goal of our work was to evaluate how different metric spaces maintain the topology of databases without any prior incentive to do so.  This is in contrast to previous work where TDA has been applied to hierarchical database representations and compactifications, such as work by \cite{aloni2021joint} who impose geometric and topological structures and work by \cite{Moor2020} who explicitly enforce a topological constraint.  Our motivation for not imposing any topological structure is with task of IR in mind: while ``bad" databases may need their topology to be enforced, for ``good" databases, this is not necessary.  In this work, we do not seek to improve the topological quality of the database representations.  Our results presented here are a first step in assessing the utility of TDA as a viable tool in IR, which was indeed demonstrated in Section \ref{sec:classification}.  Our findings show that without enforcing a topological prior, the topology is not preserved between two commonly-used embeddings to a high degree, which we quantify with the dilation-invariant bottleneck distance.  Where the connectivity of the database is important, i.e., when we seek to preserve cycles and higher order topology, database embeddings into the Poincar\'{e} ball are less desirable.  While embeddings into the Poincar\'{e} ball indeed preserve hierarchy and are useful when the databases are known to be trees, when it is {\em a priori} unknown if there exists a higher order structure that would be important for IR, our work shows that a more conservative approach would be to use an alternative embedding.

Our findings motivate the quantification of the added information and potential of increased accuracy in IR tasks when taking into account the topology of the hierarchy.  The initiation toward this direction shown by our inferential analysis where IR was performed by classification is promising.  Fully harnessing the potential of TDA in IR tasks will entail the development of a new set of topological quality control and performance criteria, which will need to be interpretable in context and comparable to currently used methods.  These criteria will need to be applied to both the queries submitted and the database, since queries are not always informative or well-designed, while databases may also suffer from missing or repeated entries and noise.  An extensive set of experiments is required to assess all possible scenarios (i.e., for when queries are ``bad," but the database is ``good," and vice versa; and also when both the query and database are ``bad.").  In future work, we aim to develop a topological criterion for preference of certain database embeddings over others, which may then be used to increase the precision of IR.


\section*{Acknowledgments}

The authors wish to thank Henry Adams, Th\'{e}o Lacombe, and Facundo M\'{e}moli for helpful discussions; we are also grateful to Don Sheehy for providing the code to compute the shift-invariant bottleneck distance. In addition, we would like to thank the NVIDIA Corporation for the donation of the GPUs utilized in this project.

Y.C.~is funded by a President's PhD Scholarship at Imperial College London.  B.K.'s research is supported by the UK Engineering and Physical Sciences Research Council and Innovate UK; B.K.~also receives personal fees and other from ThinkSono Ldt., Ultromics Ldt., and Cydar Medical Ldt., outside the submitted work. A.M.~wishes to acknowledge start-up funding from the Department of Mathematics at Imperial College London. L.S.'s PhD is funded by the UK Engineering and Physical Sciences Research Council under award EP/S013687/1. A.V.'s PhD is funded by the Department of Computing at Imperial College London.

\begin{sidewaystable}[htpb]
\centering

\caption{IR results: Reported Top 1st/Top 2nd prediction accuracy using the dilation-invariant bottleneck dissimilarity (DI) compared to the standard bottleneck distance (BD).  We see that our proposed dilation-invariant bottleneck dissimilarity at least matches the performance of the standard bottleneck distance and outperforms it in instances of smaller proportions of data. 
\label{tab:ir_results}}

\begin{tabular}{|c|c|c|c|c|c|c|c|c|c|}

\hline
 & Distance  & Derma      & Pneumonia & Retina     & Breast  & Organ\_Axial & Organ\_Coronal & Organ\_Sagittal \\ \hline

\multirow{2}{*}{20\%} & DI & 0.7142/0.8571  & \textbf{0.7777}/0.8888  & \textbf{0.7857/0.9285}   & \textbf{0.8125}/0.875 & \textbf{0.7407/0.8888}     & \textbf{0.6578/0.8157}       & \textbf{0.5714}/0.6938        \\ 

& BD   & 0.7142/1.0     & 0.6666/1.0    & 0.7142/0.9285 & 0.75/0.9375  & 0.6296/0.8518     & 0.5526/0.7894       & 0.4897/0.7142        \\ \hline

\multirow{2}{*}{40\%} & DI & 0.5714/0.8571 & 0.6666/\textbf{0.8888}  & \textbf{0.7857/0.9285}   & \textbf{0.8125}/0.875 & \textbf{0.8518/0.9259}     & \textbf{0.7105/0.8421}       & \textbf{0.6122}/0.7346        \\

& BD   & 0.7142/0.8571  & 0.7777/0.8888  & 0.7142/0.9285 & 0.75/0.9375  & 0.7777/0.9259     & 0.6842/0.8421       & 0.5918/0.7551        \\ \hline

\multirow{2}{*}{60\%} & DI & 0.7142/0.8571 & 0.7777/\textbf{1.0}     & 0.7142/\textbf{1.0}     & 0.75/\textbf{1.0}     & 0.7777/\textbf{0.9259}     & 0.6842/\textbf{0.8684}       & 0.59183/\textbf{0.7551}       \\

& BD   & 0.8571/1.0    & 0.8888/1.0     & 0.8571/0.9286 & 0.875/0.9375 & 0.8148/0.9259     & 0.7631/0.8684       & 0.6530/0.7551        \\ \hline

\multirow{2}{*}{80\%} & DI & 0.7142/0.8571  & \textbf{0.7777/1.0}     & 0.7142/0.9286 & 0.75/0.9375 & 0.7037/\textbf{0.9629}     & 0.6315/\textbf{0.8947}       & 0.5510/\textbf{0.7755}        \\

& BD   & 0.8571/1.0    & 0.7777/1.0     & 0.8571/1.0     & 0.875/1.0   & 0.8518/0.92592    & 0.657/0.8157        & 0.5714/0.7142        \\ \hline

\end{tabular}
\end{sidewaystable}
 







\clearpage
\newpage
\bibliographystyle{chicago}  
\bibliography{probe_ref} 

\begin{thebibliography}{}

\bibitem[\protect\citeauthoryear{Adams and Carlsson}{Adams and
  Carlsson}{2015}]{doi:10.1177/0278364914548051}
Adams, H. and G.~Carlsson (2015).
\newblock Evasion paths in mobile sensor networks.
\newblock {\em The International Journal of Robotics Research\/}~{\em 34\/}(1),
  90--104.

\bibitem[\protect\citeauthoryear{Adams, Emerson, Kirby, Neville, Peterson,
  Shipman, Chepushtanova, Hanson, Motta, and Ziegelmeier}{Adams
  et~al.}{2017}]{adams2017persistence}
Adams, H., T.~Emerson, M.~Kirby, R.~Neville, C.~Peterson, P.~Shipman,
  S.~Chepushtanova, E.~Hanson, F.~Motta, and L.~Ziegelmeier (2017).
\newblock Persistence images: {A} stable vector representation of persistent
  homology.
\newblock {\em Journal of Machine Learning Research\/}~{\em 18}.

\bibitem[\protect\citeauthoryear{Aloni, Bobrowski, and Talmon}{Aloni
  et~al.}{2021}]{aloni2021joint}
Aloni, L., O.~Bobrowski, and R.~Talmon (2021).
\newblock Joint {G}eometric and {T}opological {A}nalysis of {H}ierarchical
  {D}atasets.
\newblock {\em arXiv preprint arXiv:2104.01395\/}.

\bibitem[\protect\citeauthoryear{Anderson, Anderson, Palande, and
  Wang}{Anderson et~al.}{2018}]{10.1007/978-3-030-00755-3_8}
Anderson, K.~L., J.~S. Anderson, S.~Palande, and B.~Wang (2018).
\newblock Topological {D}ata {A}nalysis of {F}unctional {M}{R}{I}
  {C}onnectivity in {T}ime and {S}pace {D}omains.
\newblock In G.~Wu, I.~Rekik, M.~D. Schirmer, A.~W. Chung, and B.~Munsell
  (Eds.), {\em Connectomics in NeuroImaging}, Cham, pp.\  67--77. Springer
  International Publishing.

\bibitem[\protect\citeauthoryear{Aukerman, Carri{\`e}re, Chen, Gardner,
  Rabad{\'a}n, and Vanguri}{Aukerman
  et~al.}{2020}]{aukerman_et_al:LIPIcs:2020:12169}
Aukerman, A., M.~Carri{\`e}re, C.~Chen, K.~Gardner, R.~Rabad{\'a}n, and
  R.~Vanguri (2020).
\newblock {Persistent Homology Based Characterization of the Breast Cancer
  Immune Microenvironment: A Feasibility Study}.
\newblock In S.~Cabello and D.~Z. Chen (Eds.), {\em 36th International
  Symposium on Computational Geometry (SoCG 2020)}, Volume 164 of {\em Leibniz
  International Proceedings in Informatics (LIPIcs)}, Dagstuhl, Germany, pp.\
  11:1--11:20. Schloss Dagstuhl--Leibniz-Zentrum f{\"u}r Informatik.

\bibitem[\protect\citeauthoryear{Bauer}{Bauer}{2021}]{bauer2021ripser}
Bauer, U. (2021).
\newblock Ripser: efficient computation of vietoris--rips persistence barcodes.
\newblock {\em Journal of Applied and Computational Topology\/}~{\em 5\/}(3),
  391--423.

\bibitem[\protect\citeauthoryear{Bevilacqua and Navigli}{Bevilacqua and
  Navigli}{2020}]{bevilacqua-navigli-2020-breaking}
Bevilacqua, M. and R.~Navigli (2020).
\newblock Breaking through the 80{\%} glass ceiling: {R}aising the state of the
  art in word sense disambiguation by incorporating knowledge graph
  information.
\newblock In {\em Proceedings of the 58th Annual Meeting of the Association for
  Computational Linguistics}, Online, pp.\  2854--2864. Association for
  Computational Linguistics.

\bibitem[\protect\citeauthoryear{Bobrowski, Kahle, and Skraba}{Bobrowski
  et~al.}{2017}]{bobrowski2017}
Bobrowski, O., M.~Kahle, and P.~Skraba (2017, 08).
\newblock Maximally persistent cycles in random geometric complexes.
\newblock {\em Ann. Appl. Probab.\/}~{\em 27\/}(4), 2032--2060.

\bibitem[\protect\citeauthoryear{Boudin, Gallina, and Aizawa}{Boudin
  et~al.}{2020}]{boudin-etal-2020-keyphrase}
Boudin, F., Y.~Gallina, and A.~Aizawa (2020, July).
\newblock Keyphrase generation for scientific document retrieval.
\newblock In {\em Proceedings of the 58th Annual Meeting of the Association for
  Computational Linguistics}, Online, pp.\  1118--1126. Association for
  Computational Linguistics.

\bibitem[\protect\citeauthoryear{Br{\"{u}}el-Gabrielsson, Nelson, Dwaraknath,
  Skraba, Guibas, and Carlsson}{Br{\"{u}}el-Gabrielsson
  et~al.}{2019}]{Bruel-Gabrielsson2019}
Br{\"{u}}el-Gabrielsson, R., B.~J. Nelson, A.~Dwaraknath, P.~Skraba, L.~J.
  Guibas, and G.~Carlsson (2019).
\newblock {A Topology Layer for Machine Learning}.

\bibitem[\protect\citeauthoryear{Buchet, Chazal, Oudot, and Sheehy}{Buchet
  et~al.}{2016}]{BUCHET201670}
Buchet, M., F.~Chazal, S.~Y. Oudot, and D.~R. Sheehy (2016).
\newblock Efficient and robust persistent homology for measures.
\newblock {\em Computational Geometry\/}~{\em 58}, 70--96.

\bibitem[\protect\citeauthoryear{Burago, Burago, Burago, Ivanov, Ivanov, and
  Ivanov}{Burago et~al.}{2001}]{burago2001course}
Burago, D., I.~D. Burago, Y.~Burago, S.~Ivanov, S.~V. Ivanov, and S.~A. Ivanov
  (2001).
\newblock {\em A course in metric geometry}, Volume~33.
\newblock American Mathematical Soc.

\bibitem[\protect\citeauthoryear{Cao and Monod}{Cao and
  Monod}{2022}]{cao2022approximating}
Cao, Y. and A.~Monod (2022).
\newblock Approximating persistent homology for large datasets.
\newblock {\em arXiv preprint arXiv:2204.09155\/}.

\bibitem[\protect\citeauthoryear{Carlsson}{Carlsson}{2009}]{carlsson2009topology}
Carlsson, G. (2009).
\newblock Topology and {D}ata.
\newblock {\em Bulletin of the American Mathematical Society\/}~{\em 46\/}(2),
  255--308.

\bibitem[\protect\citeauthoryear{Chazal, Cohen-Steiner, Guibas, M\'{e}moli, and
  Oudot}{Chazal
  et~al.}{2009}]{https://doi.org/10.1111/j.1467-8659.2009.01516.x}
Chazal, F., D.~Cohen-Steiner, L.~J. Guibas, F.~M\'{e}moli, and S.~Y. Oudot
  (2009).
\newblock Gromov--{H}ausdorff {S}table {S}ignatures for {S}hapes using
  {P}ersistence.
\newblock {\em Computer Graphics Forum\/}~{\em 28\/}(5), 1393--1403.

\bibitem[\protect\citeauthoryear{Chazal, Fasy, Lecci, Michel, Rinaldo, and
  Wasserman}{Chazal et~al.}{2015}]{chazal2015subsampling}
Chazal, F., B.~Fasy, F.~Lecci, B.~Michel, A.~Rinaldo, and L.~Wasserman (2015).
\newblock Subsampling methods for persistent homology.
\newblock In {\em International Conference on Machine Learning}, pp.\
  2143--2151. PMLR.

\bibitem[\protect\citeauthoryear{Chazal, Guibas, Oudot, and Skraba}{Chazal
  et~al.}{2013}]{10.1145/2535927}
Chazal, F., L.~J. Guibas, S.~Y. Oudot, and P.~Skraba (2013).
\newblock Persistence-{B}ased {C}lustering in {R}iemannian {M}anifolds.
\newblock {\em J. ACM\/}~{\em 60\/}(6).

\bibitem[\protect\citeauthoryear{Chung, Hu, Lo, and Wu}{Chung
  et~al.}{2021}]{10.3389/fphys.2021.637684}
Chung, Y.-M., C.-S. Hu, Y.-L. Lo, and H.-T. Wu (2021).
\newblock A persistent homology approach to heart rate variability analysis
  with an application to sleep-wake classification.
\newblock {\em Frontiers in Physiology\/}~{\em 12}, 202.

\bibitem[\protect\citeauthoryear{Clementini, Sharma, and Egenhofer}{Clementini
  et~al.}{1994}]{CLEMENTINI1994815}
Clementini, E., J.~Sharma, and M.~J. Egenhofer (1994).
\newblock Modelling topological spatial relations: Strategies for query
  processing.
\newblock {\em Computers \& Graphics\/}~{\em 18\/}(6), 815--822.

\bibitem[\protect\citeauthoryear{Cohen-Steiner, Edelsbrunner, and
  Harer}{Cohen-Steiner et~al.}{2007}]{cohen2007stability}
Cohen-Steiner, D., H.~Edelsbrunner, and J.~Harer (2007).
\newblock Stability of persistence diagrams.
\newblock {\em Discrete \& computational geometry\/}~{\em 37\/}(1), 103--120.

\bibitem[\protect\citeauthoryear{Crawford, Monod, Chen, Mukherjee, and
  Rabad\'{a}n}{Crawford et~al.}{2020}]{SECT}
Crawford, L., A.~Monod, A.~X. Chen, S.~Mukherjee, and R.~Rabad\'{a}n (2020).
\newblock Predicting {C}linical {O}utcomes in {G}lioblastoma: {A}n
  {A}pplication of {T}opological and {F}unctional {D}ata {A}nalysis.
\newblock {\em Journal of the American Statistical Association\/}~{\em
  115\/}(531), 1139--1150.

\bibitem[\protect\citeauthoryear{De~Gregorio, Fugacci, M\'{e}moli, and
  Vaccarino}{De~Gregorio et~al.}{2020}]{degregorio2020notion}
De~Gregorio, A., U.~Fugacci, F.~M\'{e}moli, and F.~Vaccarino (2020).
\newblock On the notion of weak isometry for finite metric spaces.

\bibitem[\protect\citeauthoryear{de~Silva and Ghrist}{de~Silva and
  Ghrist}{2006}]{doi:10.1177/0278364906072252}
de~Silva, V. and R.~Ghrist (2006).
\newblock Coordinate-free coverage in sensor networks with controlled
  boundaries via homology.
\newblock {\em The International Journal of Robotics Research\/}~{\em
  25\/}(12), 1205--1222.

\bibitem[\protect\citeauthoryear{Deolalikar}{Deolalikar}{2015}]{enterprise-topology}
Deolalikar, V. (2015, 01).
\newblock Topological models of document-query sets in retrieval for enterprise
  information management.
\newblock pp.\  18--23.

\bibitem[\protect\citeauthoryear{Divol and Lacombe}{Divol and
  Lacombe}{2021}]{divol2021understanding}
Divol, V. and T.~Lacombe (2021).
\newblock Understanding the topology and the geometry of the space of
  persistence diagrams via optimal partial transport.
\newblock {\em Journal of Applied and Computational Topology\/}~{\em 5\/}(1),
  1--53.

\bibitem[\protect\citeauthoryear{Dunaeva, Edelsbrunner, Lukyanov, Machin,
  Malkova, Kuvaev, and Kashin}{Dunaeva et~al.}{2016}]{DUNAEVA201613}
Dunaeva, O., H.~Edelsbrunner, A.~Lukyanov, M.~Machin, D.~Malkova, R.~Kuvaev,
  and S.~Kashin (2016).
\newblock The classification of endoscopy images with persistent homology.
\newblock {\em Pattern Recognition Letters\/}~{\em 83}, 13--22.
\newblock Geometric, topological and harmonic trends to image processing.

\bibitem[\protect\citeauthoryear{Edelsbrunner and Harer}{Edelsbrunner and
  Harer}{2008}]{edelsbrunner2008persistent}
Edelsbrunner, H. and J.~Harer (2008).
\newblock Persistent {H}omology -- a {S}urvey.
\newblock {\em Contemporary mathematics\/}~{\em 453}, 257--282.

\bibitem[\protect\citeauthoryear{{Edelsbrunner}, {Letscher}, and
  {Zomorodian}}{{Edelsbrunner} et~al.}{2000}]{892133}
{Edelsbrunner}, H., D.~{Letscher}, and A.~{Zomorodian} (2000).
\newblock Topological persistence and simplification.
\newblock In {\em Proceedings 41st Annual Symposium on Foundations of Computer
  Science}, pp.\  454--463.

\bibitem[\protect\citeauthoryear{Efrat, Itai, and Katz}{Efrat
  et~al.}{2001}]{efrat2001}
Efrat, A., A.~Itai, and M.~J. Katz (2001).
\newblock Geometry helps in bottleneck matching and related problems.
\newblock {\em Algorithmica\/}~{\em 31\/}(1), 1--28.

\bibitem[\protect\citeauthoryear{Egghe}{Egghe}{1998}]{EGGHE199861}
Egghe, L. (1998).
\newblock Properties of topologies of information retrieval systems.
\newblock {\em Mathematical and Computer Modelling\/}~{\em 27\/}(2), 61--79.

\bibitem[\protect\citeauthoryear{Egghe and Rousseau}{Egghe and
  Rousseau}{1998}]{https://doi.org/10.1002/(SICI)1097-4571(1998)49:13<1144::AID-ASI2>3.0.CO;2-Z}
Egghe, L. and R.~Rousseau (1998).
\newblock Topological aspects of information retrieval.
\newblock {\em Journal of the American Society for Information Science\/}~{\em
  49\/}(13), 1144--1160.

\bibitem[\protect\citeauthoryear{Everett and Cater}{Everett and
  Cater}{1992a}]{https://doi.org/10.1002/(SICI)1097-4571(199212)43:10<658::AID-ASI3>3.0.CO;2-H}
Everett, D.~M. and S.~C. Cater (1992a).
\newblock Topology of document retrieval systems.
\newblock {\em Journal of the American Society for Information Science\/}~{\em
  43\/}(10), 658--673.

\bibitem[\protect\citeauthoryear{Everett and Cater}{Everett and
  Cater}{1992b}]{Everett1992}
Everett, D.~M. and S.~C. Cater (1992b).
\newblock {Topology of document retrieval systems}.
\newblock {\em Journal of the American Society for Information Science\/}~{\em
  43\/}(10), 658--673.

\bibitem[\protect\citeauthoryear{Frosini and Landi}{Frosini and
  Landi}{1999}]{frosini1999size}
Frosini, P. and C.~Landi (1999).
\newblock Size theory as a topological tool for computer vision.
\newblock {\em Pattern Recognition and Image Analysis\/}~{\em 9\/}(4),
  596--603.

\bibitem[\protect\citeauthoryear{Ganea, B{\'{e}}cigneul, and Hofmann}{Ganea
  et~al.}{2018}]{Ganea2018networks}
Ganea, O.~E., G.~B{\'{e}}cigneul, and T.~Hofmann (2018).
\newblock {Hyperbolic neural networks}.
\newblock {\em Advances in Neural Information Processing Systems\/}~{\em
  2018\/}(NeurIPS), 5345--5355.

\bibitem[\protect\citeauthoryear{Ghrist}{Ghrist}{2008}]{ghrist2008barcodes}
Ghrist, R. (2008).
\newblock Barcodes: {T}he persistent topology of data.
\newblock {\em Bulletin of the American Mathematical Society\/}~{\em 45\/}(1),
  61--75.

\bibitem[\protect\citeauthoryear{Google}{Google}{2018}]{kaggle}
Google (2018).
\newblock Kaggle {G}oogle {L}andmark {R}etrieval {C}hallenge.

\bibitem[\protect\citeauthoryear{Heilbron, Escorcia, Ghanem, and
  Niebles}{Heilbron et~al.}{2015}]{caba2015activitynet}
Heilbron, F.~C., V.~Escorcia, B.~Ghanem, and J.~C. Niebles (2015).
\newblock Activity{N}et: {A} large-scale video benchmark for human activity
  understanding.
\newblock In {\em Proceedings of the IEEE Conference on Computer Vision and
  Pattern Recognition}, pp.\  961--970.

\bibitem[\protect\citeauthoryear{Hiraoka, Nakamura, Hirata, Escolar, Matsue,
  and Nishiura}{Hiraoka et~al.}{2016}]{Hiraoka201520877}
Hiraoka, Y., T.~Nakamura, A.~Hirata, E.~G. Escolar, K.~Matsue, and Y.~Nishiura
  (2016).
\newblock Hierarchical structures of amorphous solids characterized by
  persistent homology.
\newblock {\em Proceedings of the National Academy of Sciences\/}.

\bibitem[\protect\citeauthoryear{Hirata, Wada, Obayashi, and Hiraoka}{Hirata
  et~al.}{2020}]{hirata2020structural}
Hirata, A., T.~Wada, I.~Obayashi, and Y.~Hiraoka (2020).
\newblock Structural changes during glass formation extracted by computational
  homology with machine learning.
\newblock {\em Communications Materials\/}~{\em 1\/}(1), 1--8.

\bibitem[\protect\citeauthoryear{Hofer, Kwitt, Niethammer, and Dixit}{Hofer
  et~al.}{2019}]{pmlr-v97-hofer19a}
Hofer, C., R.~Kwitt, M.~Niethammer, and M.~Dixit (2019).
\newblock Connectivity-{O}ptimized {R}epresentation {L}earning via {P}ersistent
  {H}omology.
\newblock In K.~Chaudhuri and R.~Salakhutdinov (Eds.), {\em Proceedings of the
  36th International Conference on Machine Learning}, Volume~97 of {\em
  Proceedings of Machine Learning Research}, pp.\  2751--2760. PMLR.

\bibitem[\protect\citeauthoryear{Hofer, Kwitt, Niethammer, and Uhl}{Hofer
  et~al.}{2017}]{hofer2017deep}
Hofer, C., R.~Kwitt, M.~Niethammer, and A.~Uhl (2017).
\newblock Deep learning with topological signatures.
\newblock {\em arXiv preprint arXiv:1707.04041\/}.

\bibitem[\protect\citeauthoryear{Hopcroft and Karp}{Hopcroft and
  Karp}{1971}]{hk1973}
Hopcroft, J.~E. and R.~M. Karp (1971).
\newblock A $n^{5/2}$ algorithm for maximum matchings in bipartite.
\newblock In {\em Proceedings of the 12th Annual Symposium on Switching and
  Automata Theory (Swat 1971)}, SWAT '71, USA, pp.\  122–125. IEEE Computer
  Society.

\bibitem[\protect\citeauthoryear{Hou}{Hou}{2019}]{ResNetAE}
Hou, B. (2019).
\newblock {R}es{N}et{AE}-https://github.com/farrell236/resnetae.

\bibitem[\protect\citeauthoryear{Hu, Fuxin, Samaras, and Chen}{Hu
  et~al.}{2019}]{Hu2019}
Hu, X., L.~Fuxin, D.~Samaras, and C.~Chen (2019).
\newblock {Topology-Preserving Deep Image Segmentation}.
\newblock pp.\  1--11.

\bibitem[\protect\citeauthoryear{Kerber, Morozov, and Nigmetov}{Kerber
  et~al.}{2016}]{kerber52hera}
Kerber, M., D.~Morozov, and A.~Nigmetov (2016).
\newblock Hera.
\newblock {\em URL: https://bitbucket. org/grey\_narn/hera\/}.

\bibitem[\protect\citeauthoryear{Kerber, Morozov, and Nigmetov}{Kerber
  et~al.}{2017}]{morozov2017}
Kerber, M., D.~Morozov, and A.~Nigmetov (2017).
\newblock Geometry helps to compare persistence diagrams.
\newblock {\em ACM Journal of Experimental Algorithmics\/}~{\em 22}, 1--20.

\bibitem[\protect\citeauthoryear{Lacombe, Cuturi, and Oudot}{Lacombe
  et~al.}{2018}]{lacombe2018large}
Lacombe, T., M.~Cuturi, and S.~Oudot (2018).
\newblock Large scale computation of means and clusters for persistence
  diagrams using optimal transport.
\newblock In {\em Proceedings of the 32nd International Conference on Neural
  Information Processing Systems}, pp.\  9792--9802.

\bibitem[\protect\citeauthoryear{Lawson, Schupbach, Fasy, and Sheppard}{Lawson
  et~al.}{2019}]{prostate}
Lawson, P., J.~Schupbach, B.~T. Fasy, and J.~W. Sheppard (2019).
\newblock {Persistent homology for the automatic classification of prostate
  cancer aggressiveness in histopathology images}.
\newblock In J.~E. Tomaszewski and A.~D. Ward (Eds.), {\em Medical Imaging
  2019: Digital Pathology}, Volume 10956, pp.\  72 -- 85. International Society
  for Optics and Photonics: SPIE.

\bibitem[\protect\citeauthoryear{Liu, Zhang, Goldwasser, and Wang}{Liu
  et~al.}{2020}]{liu-etal-2020-cross-lingual-document}
Liu, J., X.~Zhang, D.~Goldwasser, and X.~Wang (2020, December).
\newblock Cross-lingual document retrieval with smooth learning.
\newblock In {\em Proceedings of the 28th International Conference on
  Computational Linguistics}, Barcelona, Spain (Online), pp.\  3616--3629.
  International Committee on Computational Linguistics.

\bibitem[\protect\citeauthoryear{Long, Mettes, Shen, and Snoek}{Long
  et~al.}{2020}]{Long_2020_CVPR}
Long, T., P.~Mettes, H.~T. Shen, and C.~G.~M. Snoek (2020, June).
\newblock Searching for actions on the hyperbole.
\newblock In {\em Proceedings of the IEEE/CVF Conference on Computer Vision and
  Pattern Recognition (CVPR)}.

\bibitem[\protect\citeauthoryear{Mathieu, Lan, Maddison, Tomioka, and
  Teh}{Mathieu et~al.}{2019}]{Mathieu2019}
Mathieu, E., C.~L. Lan, C.~J. Maddison, R.~Tomioka, and Y.~W. Teh (2019).
\newblock {Continuous Hierarchical Representations with Poincar\'{e}
  Variational Auto-Encoders}.

\bibitem[\protect\citeauthoryear{Miller}{Miller}{1995}]{10.1145/219717.219748}
Miller, G.~A. (1995).
\newblock Word{N}et: A lexical database for english.
\newblock {\em Commun. ACM\/}~{\em 38\/}(11), 39--–41.

\bibitem[\protect\citeauthoryear{Moor, Horn, Rieck, and Borgwardt}{Moor
  et~al.}{2020}]{Moor2020}
Moor, M., M.~Horn, B.~Rieck, and K.~Borgwardt (2020).
\newblock {Topological Autoencoders}.
\newblock {\em Proceedings of the 37th International Conference on Machine
  Learning\/}, 1--18.

\bibitem[\protect\citeauthoryear{Munkres}{Munkres}{2018}]{munkres2018elements}
Munkres, J.~R. (2018).
\newblock {\em Elements of algebraic topology}.
\newblock CRC press.

\bibitem[\protect\citeauthoryear{Nathaniel~Saul}{Nathaniel~Saul}{2019}]{scikittda2019}
Nathaniel~Saul, C.~T. (2019).
\newblock Scikit-{TDA}: {T}opological {D}ata {A}nalysis for {P}ython.

\bibitem[\protect\citeauthoryear{Nickel and Kiela}{Nickel and
  Kiela}{2017}]{Nickel2017}
Nickel, M. and D.~Kiela (2017).
\newblock {Poincar{\'{e}} embeddings for learning hierarchical
  representations}.
\newblock {\em Advances in Neural Information Processing Systems\/}~{\em
  2017\/}(NeurIPS), 6339--6348.

\bibitem[\protect\citeauthoryear{Otter, Porter, Tillmann, Grindrod, and
  Harrington}{Otter et~al.}{2017}]{otter2017roadmap}
Otter, N., M.~A. Porter, U.~Tillmann, P.~Grindrod, and H.~A. Harrington (2017).
\newblock A roadmap for the computation of persistent homology.
\newblock {\em EPJ Data Science\/}~{\em 6}, 1--38.

\bibitem[\protect\citeauthoryear{Patania, Selvaggi, Veronese, Dipasquale,
  Expert, and Petri}{Patania et~al.}{2019}]{10.1162/netn_a_00094}
Patania, A., P.~Selvaggi, M.~Veronese, O.~Dipasquale, P.~Expert, and G.~Petri
  (2019).
\newblock {Topological gene expression networks recapitulate brain anatomy and
  function}.
\newblock {\em Network Neuroscience\/}~{\em 3\/}(3), 744--762.

\bibitem[\protect\citeauthoryear{Perea and Carlsson}{Perea and
  Carlsson}{2014}]{perea2014klein}
Perea, J.~A. and G.~Carlsson (2014).
\newblock A {K}lein-{B}ottle-{B}ased {D}ictionary for {T}exture
  {R}epresentation.
\newblock {\em International journal of computer vision\/}~{\em 107\/}(1),
  75--97.

\bibitem[\protect\citeauthoryear{Pokorny, Hawasly, and Ramamoorthy}{Pokorny
  et~al.}{2014}]{pokorny2014multiscale}
Pokorny, F.~T., M.~Hawasly, and S.~Ramamoorthy (2014).
\newblock Multiscale topological trajectory classification with persistent
  homology.
\newblock In {\em Robotics: science and systems}.

\bibitem[\protect\citeauthoryear{Reininghaus, Huber, Bauer, and
  Kwitt}{Reininghaus et~al.}{2015}]{reininghaus2015stable}
Reininghaus, J., S.~Huber, U.~Bauer, and R.~Kwitt (2015).
\newblock A stable multi-scale kernel for topological machine learning.
\newblock In {\em Proceedings of the IEEE conference on computer vision and
  pattern recognition}, pp.\  4741--4748.

\bibitem[\protect\citeauthoryear{Scarlini, Pasini, and Navigli}{Scarlini
  et~al.}{2020}]{scarlini-etal-2020-sensembert}
Scarlini, B., T.~Pasini, and R.~Navigli (2020).
\newblock {SensEmBERT: Context-Enhanced Sense Embeddings for Multilingual Word
  Sense Disambiguation}.
\newblock In {\em Proceedings of the Thirty-Fourth Conference on Artificial
  Intelligence}, pp.\  8758--8765. Association for the Advancement of
  Artificial Intelligence.

\bibitem[\protect\citeauthoryear{Sheehy, Kisielius, and Cavanna}{Sheehy
  et~al.}{2018}]{sheehy2018computing}
Sheehy, D., O.~Kisielius, and N.~J. Cavanna (2018).
\newblock Computing the {S}hift-{I}nvariant {B}ottleneck {D}istance for
  {P}ersistence {D}iagrams.
\newblock In {\em CCCG}, pp.\  78--84.

\bibitem[\protect\citeauthoryear{Tauzin, Lupo, Tunstall, Pérez, Caorsi,
  Medina-Mardones, Dassatti, and Hess}{Tauzin
  et~al.}{2020}]{tauzin2020giottotda}
Tauzin, G., U.~Lupo, L.~Tunstall, J.~B. Pérez, M.~Caorsi, A.~Medina-Mardones,
  A.~Dassatti, and K.~Hess (2020).
\newblock giotto-tda: A topological data analysis toolkit for machine learning
  and data exploration.

\bibitem[\protect\citeauthoryear{Teramoto, Shinohara, and Takiyama}{Teramoto
  et~al.}{2020}]{TERAMOTO2020105614}
Teramoto, T., T.~Shinohara, and A.~Takiyama (2020).
\newblock Computer-aided classification of hepatocellular ballooning in liver
  biopsies from patients with nash using persistent homology.
\newblock {\em Computer Methods and Programs in Biomedicine\/}~{\em 195},
  105614.

\bibitem[\protect\citeauthoryear{{The GUDHI Project}}{{The GUDHI
  Project}}{2021}]{gudhi:urm}
{The GUDHI Project} (2021).
\newblock {\em {GUDHI} {U}ser and {R}eference {M}anual\/} ({3.4.1} ed.).
\newblock {GUDHI Editorial Board}.

\bibitem[\protect\citeauthoryear{Turner, Mileyko, Mukherjee, and Harer}{Turner
  et~al.}{2014}]{turner2014frechet}
Turner, K., Y.~Mileyko, S.~Mukherjee, and J.~Harer (2014).
\newblock Fr{\'e}chet means for distributions of persistence diagrams.
\newblock {\em Discrete \& Computational Geometry\/}~{\em 52\/}(1), 44--70.

\bibitem[\protect\citeauthoryear{Yang, Shi, and Ni}{Yang
  et~al.}{2021}]{medmnistv1}
Yang, J., R.~Shi, and B.~Ni (2021).
\newblock Med{MNIST} {C}lassification {D}ecathlon: {A} {L}ightweight {A}uto{ML}
  {B}enchmark for {M}edical {I}mage {A}nalysis.
\newblock In {\em IEEE 18th International Symposium on Biomedical Imaging
  (ISBI)}, pp.\  191--195.

\bibitem[\protect\citeauthoryear{Yap, Koh, and Chng}{Yap
  et~al.}{2020}]{yap-etal-2020-adapting}
Yap, B.~P., A.~Koh, and E.~S. Chng (2020).
\newblock Adapting {BERT} for word sense disambiguation with gloss selection
  objective and example sentences.
\newblock In {\em Findings of the Association for Computational Linguistics:
  EMNLP 2020}, Online, pp.\  41--46. Association for Computational Linguistics.

\bibitem[\protect\citeauthoryear{Zomorodian and Carlsson}{Zomorodian and
  Carlsson}{2005}]{zomorodian2005computing}
Zomorodian, A. and G.~Carlsson (2005).
\newblock Computing persistent homology.
\newblock {\em Discrete \& Computational Geometry\/}~{\em 33\/}(2), 249--274.

\end{thebibliography}


\end{document}